\documentclass[nohyperref]{article}
\usepackage[preprint]{neurips_2022}

\usepackage{amsmath,amssymb,amsthm,amsfonts,mathrsfs}
\usepackage{upgreek}
\usepackage{nicefrac}
\usepackage{hyperref}
\DeclareMathAlphabet{\mathpzc}{OT1}{pzc}{m}{it}
\hypersetup{
 colorlinks = true,
 linkcolor=black,
 citecolor = blue,
}
\DeclareUrlCommand\url{\color{black}}

\usepackage{color}
\usepackage{stmaryrd}
\usepackage{enumitem}
\usepackage{xcolor}
\usepackage{bbm}
\usepackage{ifthen}
\usepackage{xargs}
\usepackage[disable]{todonotes}
\usepackage{aliascnt}
\usepackage{tikz-cd}

\usepackage{subfigure}
\usepackage{graphicx}
\usepackage{comment}
\usepackage{array}
\usepackage{cleveref}
\usepackage{autonum}
\usepackage{tabularx}
\usepackage{float}
\usepackage{algorithm,algorithmic}

\newtheorem{theorem}{Theorem}
\crefname{theorem}{theorem}{Theorems}
\Crefname{Theorem}{Theorem}{Theorems}
\theoremstyle{definition}

\newtheorem*{lemma_nonumber*}{Lemma}
\newaliascnt{lemma}{theorem}
\newtheorem{lemma}[lemma]{Lemma}
\aliascntresetthe{lemma}
\crefname{lemma}{lemma}{lemmas}
\Crefname{Lemma}{Lemma}{Lemmas}
\newaliascnt{corollary}{theorem}

\aliascntresetthe{corollary}
\crefname{corollary}{corollary}{corollaries}
\Crefname{Corollary}{Corollary}{Corollaries}
\newaliascnt{proposition}{theorem}
\newtheorem{proposition}[proposition]{Proposition}
\aliascntresetthe{proposition}
\crefname{proposition}{proposition}{propositions}
\Crefname{Proposition}{Proposition}{Propositions}
\newaliascnt{remark}{theorem}

\aliascntresetthe{remark}
\crefname{remark}{remark}{remarks}
\Crefname{Remark}{Remark}{Remarks}
\newtheorem{assumption}{\textbf{A}\hspace{-3pt}}
\crefformat{assumption}{{\textbf{A}}#2#1#3}

\crefformat{assumptionB}{{\textbf{B}}#2#1#3}

\crefformat{assumptionD}{{\textbf{D}}#2#1#3}

\newcommand{\ps}[1]{\langle #1 \rangle}

\newcommand{\cP}{{\mathcal P}}

\newcommand{\cN}{{\mathcal N}}

\newcommand{\cD}{{\mathcal D}}

\newcommand{\X}{{\mathsf X}}
\newcommand{\Y}{{\mathsf Y}}
\newcommand{\R}{{\mathbb R}}
\newcommand{\E}{{\mathbb E}}
\newcommand{\D}{{\mathcal D}}
\newcommand{\N}{{\mathbb N}}

\newcommand{\prior}{P_0}

\newcommand{\diag}{\mathop{\mathrm{Diag}}\nolimits}

\newcommand{\btheta}{\boldsymbol{\theta}}
\newcommand{\bw}{\boldsymbol{w}}
\newcommand{\bz}{\boldsymbol{z}}
\newcommand{\bmu}{\boldsymbol{\mu}}
\newcommand{\brho}{\boldsymbol{\rho}}

\newcommand{\bgamma}{\boldsymbol{\gamma}}

\def\ie{i.e.}

\def\likelihood{\mathrm{L}}
\def\loss{\ell}
\def\lossCE{\ell_{\mathrm{CE}}}

\def\diag{\mathrm{diag}}

\newcommandx{\norm}[2][1=]{\ifthenelse{\equal{#1}{}}{\left\Vert #2 \right\Vert}{\left\Vert #2 \right\Vert^{#1}}}
\newcommandx{\normLigne}[2][1=]{\ifthenelse{\equal{#1}{}}{\Vert #2 \Vert}{\Vert #2\Vert^{#1}}}

\def\argmax{\operatorname{argmax}}

\def\rset{\mathbb{R}}

\newcommandx{\propmap}[1 ]{\mathrm{T}_{#1}}
\newcommandx{\chunk}[3]{#1_{#2:#3}}

\newcommandx{\Circ}[2]{\underset{#1}{\overset{#2}{\bigcirc}}}
\newcommandx{\Circtext}[2]{{\bigcirc}_{#1}^{#2}}

\def\bw{\bar{w}}

\def\KL{\mathrm{KL}}

\def\mse{\mathsf{E}}

\def\msx{\mathsf{X}}
\def\msy{\mathsf{Y}}
\def\erf{\mathsf{erf}}

\def\mce{\mathcal{E}}

\def\mcp{\mathcal{P}}

\def\rset{\mathbb{R}}

\def\nsets{\mathbb{N}^*}

\def\rmd{\mathrm{d}}

\def\rme{\mathrm{e}}

\def\rmC{\mathrm{C}}

\newcommandx{\functionspace}[2][1=+]{\mathbb{F}_{#1}(#2)}

\newcommand{\LeftEqNo}{\let\veqno\@@leqno}

\newcommand{\abs}[1]{\left\vert #1 \right\vert}

\newcommandx{\Vnorm}[2][1=V]{\| #2 \|_{#1}}
\newcommandx{\VnormEq}[2][1=V]{\left\| #2 \right\|_{#1}}

\newcommandx\probaMarkovTilde[2][2=]
{\ifthenelse{\equal{#2}{}}{{\widetilde{\mathbb{P}}_{#1}}}{\widetilde{\mathbb{P}}_{#1}\left[ #2\right]}}

\newcommand{\plusinfty}{+\infty}

\def\ie{\textit{i.e.}}

\def\eqsp{\;}

\newcommandx\sequence[3][2=,3=]
{\ifthenelse{\equal{#3}{}}{\ensuremath{\{ #1_{#2}\}}}{\ensuremath{\{ #1_{#2}, \eqsp #2 \in #3 \}}}}
\newcommandx\sequenceD[3][2=,3=]
{\ifthenelse{\equal{#3}{}}{\ensuremath{\{ #1_{#2}\}}}{\ensuremath{( #1)_{ #2 \in #3} }}}

\newcommandx{\sequencen}[2][2=n\in\N]{\ensuremath{\{ #1_n, \eqsp #2 \}}}
\newcommandx\sequenceDouble[4][3=,4=]
{\ifthenelse{\equal{#3}{}}{\ensuremath{\{ (#1_{#3},#2_{#3}) \}}}{\ensuremath{\{  (#1_{#3},#2_{#3}), \eqsp #3 \in #4 \}}}}
\newcommandx{\sequencenDouble}[3][3=n\in\N]{\ensuremath{\{ (#1_{n},#2_{n}), \eqsp #3 \}}}

\def\iid{\text{i.i.d.}}

\def\eg{e.g.}

\newcommand{\opnorm}[1]{{\left\vert\kern-0.25ex\left\vert\kern-0.25ex\left\vert #1
    \right\vert\kern-0.25ex\right\vert\kern-0.25ex\right\vert}}

\newcommandx{\CPE}[3][1=]{{\mathbb E}_{#1}\left[#2 \left \vert #3 \right. \right]} 
\newcommandx{\CPVar}[3][1=]{\mathrm{Var}^{#3}_{#1}\left\{ #2 \right\}}
\newcommand{\CPP}[3][]
{\ifthenelse{\equal{#1}{}}{{\mathbb P}\left(\left. #2 \, \right| #3 \right)}{{\mathbb P}_{#1}\left(\left. #2 \, \right | #3 \right)}}

\def\scrT{\mathscr{T}}

\newcommandx{\osc}[2][1=]{\mathrm{osc}_{#1}(#2)}

\def\Jac{\mathrm{J}}

\newcommand\coupling[2]{\Gamma(\mu,\nu)}

\renewcommand{\geq}{\geqslant}
\renewcommand{\leq}{\leqslant}

\def\Leb{\mathrm{Leb}}

\def\bgamma{\bar{\gamma}}

\def\Ltt{\mathtt{L}}

\def\posterior{P}
\def\NELBO{\mathrm{NELBO}^N}
\def\ELBO{\mathrm{ELBO}^N}

\def\ELBOeta{\mathrm{ELBO}^N_{\eta}}
\def\densityq{q}
\def\family{\mathscr{F}}

\def\bthetas{\btheta^{\star}}
\def\densityqOne{\densityq^1}
\def\priorOne{\prior^1}
\def\ThetaOne{\Xi}
\def\Tscr{\mathscr{T}}
\def\gammaOne{\gamma}

\def\rmG{\mathrm{G}}

\def\dx{d_{\msx}}
\def\dtheta{d_{\theta}}
\def\dy{d_{\msy}}

\def\scrS{\mathscr{S}}
\def\etaN{\eta_{N}}
\def\funFN{\mathrm{F}^N_{\eta}}
\def\tfunFN{\tilde{\mathrm{F}}^N_{\eta}}

\def\trmGN{\tilde{\mathrm{G}}}

\def\etawN{\tau}
\def\Liploss{\Ltt_{\loss}}

\def\constphi{C_{\phi}}
\def\funFinftyp{\mathrm{F}^p_{\etawN}}
\def\nus{\nu_{\star}}
\def\risk{\mathrm{R}_{\tau}}
\def\pixy{\pi}
\def\constphinus{C_{\phi}^{\nus}}
\def\tphi{\tilde{\phi}}
\def\Liph{\Ltt_h}
\def\inverse{-1}
\def\riskZero{\mathrm{R}_\mu}
\def\riskZerow{\mathrm{R}_w}
\def\ELBOwN{\mathrm{ELBO}}

\newcommand{\ak}[1]{\textcolor{pink!200}{\textbf{AK:}  #1}}

\title{Variational Inference of overparameterized Bayesian Neural Networks: a theoretical and empirical study}

\author{Tom Huix\\ CMAP, Ecole Polytechnique\\ IP Paris\\
\texttt{tom.huix@polytechnique.edu} \And Szymon Majewski \\ CMAP, Ecole Polytechnique\\ IP Paris\\ \texttt{sjm.majewski@gmail.com} \And Alain Durmus \\ENS Paris-Saclay\\ Université Paris-Saclay \\ \texttt{alain.durmus@ens-paris-saclay.fr}\And Eric Moulines \\ CMAP, Ecole Polytechnique\\ IP Paris\\ \texttt{eric.moulines@polytechnique.edu}\And Anna Korba \\ ENSAE, CREST\\ IP Paris\\ \texttt{anna.korba@ensae.fr} 
}

\begin{document}

\maketitle

\begin{abstract}
  This paper studies the Variational Inference (VI) used for training
  Bayesian Neural Networks (BNN) in the overparameterized regime, \ie,
  when the number of neurons tends to infinity. More specifically, we
  consider overparameterized two-layer BNN and point out a critical
  issue in the mean-field VI training. This problem arises from the
  decomposition of the lower bound on the evidence (ELBO) into two
  terms: one corresponding to the likelihood function of the model and
  the second to the Kullback-Leibler (KL) divergence between the prior
  distribution and the variational posterior. In particular, we show
  both theoretically and empirically that there is a trade-off between
  these two terms in the overparameterized regime only when the KL is
  appropriately re-scaled with respect to the ratio between the the
  number of observations and neurons. We also illustrate our theoretical results with
  numerical experiments that highlight the critical choice of this ratio. 
\end{abstract}

\section{Introduction}
Bayesian neural networks (BNN) have gained popularity in the field of
machine learning because they promise to combine the powerful
approximation/discrimination properties of (deep) neural networks (NN)
with the decision-theoretic approach of Bayesian inference. Among the
advantages of BNN is their ability to provide uncertainty
quantification, which is a must in many fields - \eg, autonomous
driving \cite{michelmore2020uncertainty,mcallister2017concrete},
computer vision \cite{kendall2017uncertainties}, medical diagnosis
\cite{filos2019systematic}. Second, the inclusion of prior information
in some cases leads to better generalization error and calibration in
classification tasks; see \cite{jospin2020hands,izmailov2021bayesian}
and references therein.

NN can be used to build complex probabilistic models for regression and classification tasks. Given $\bw$ corresponding to the weights and bias of an NN, the network output can be used to define a (conditional) likelihood $\likelihood(\{ (x_i,y_i) \}_{i=1}^p | \bw)$ of some observed labels $\{y_i\}_{i=1}^p$, $y_i \in \msy$ associated with feature vectors $\{x_i\}_{i=1}^p$, $x_i \in \msx$.
Specifying a prior distribution for $\bw$ and applying Bayes' rule yields the posterior distribution of weights.
In the Bayesian approach, the goal is to find the predictive distribution from new feature vectors defined as an integral with respect to the posterior. One possible approach is to use Markov-Chain Monte Carlo methods - such as Hamiltonian Monte Carlo - for inference in Bayesian neural networks; \cite{neal2011mcmc,hoffman2014no,betancourt2017conceptual}. However, the challenge of scaling HMC for applications involving high-dimensional parameter space and large datasets limits its broad application; \cite{cobb2021scaling}.
Computationally cheaper MCMC methods have been proposed, see \cite{welling2011bayesian,chen2014stochastic,brosse2018promises}; but these methods yield biased estimates of posterior expectation, see \cite{izmailov2021bayesian}.
A much simpler alternative from a computational standpoint is to use Variational Inference (VI)
\cite{blundell2015weight,gal2016dropout,louizos2017multiplicative,khan2018fast}, which approximates the posterior with a parametric distribution. Nevertheless, little is known about the validity or limitations of the latter approach, including the choice of prior and variational family and their interplay.

A number of recent papers have investigated the limiting behavior of gradient descent type algorithms for one or two hidden layers in the overparameterized regime,
\cite{chizat2018global,rotskoff2019global,mei2018mean,tzen2019neural,de2020quantitative}, \ie, the number of hidden neurons goes to  infinity.  More specifically, it was found that the gradient descent applied to risk minimization can be viewed as a temporal and spatial discretization of the Wasserstein gradient flow of a limiting function, which can be represented in the space of probability distributions over the parameters by
\begin{equation}
 \label{eq:7}
 \riskZero(\mu) = \int \loss(y,\int s(\bw,x) \rmd \mu(\bw))  \rmd \pi (x,y) + \mathrm{N}(\mu) \eqsp,
\end{equation}
where $\pi$ is the data distribution over $\msx \times \msy$,
$s(\bw,x)$ is the output prediction of the NN with
parameter weights $\bw$ and $\mathrm{N}$ plays the role of a penalty
function. Roughly speaking, identifying this function consists in
noting that the risk $\riskZerow$ over the weights $\bw$ of a NN
coincides with $\riskZero$ on the set of empirical measures, \ie~for
any $\bw = (w_1,\ldots,w_N)$ - where
$N$ is the number of neurons-, $\riskZerow(\bw) = \riskZero(\mu_N)$
with $\mu_N=N^{-1} \sum_{i=1}^N \updelta_{w_i}$.  This result
emphasizes that in the overparameterized regime, the parameter weights
of a NN act as particle discretization of probability measures and the
final prediction of a NN has a form of continuous mixture.

We are interested here in performing the same type of
analysis but for Variational Inference (VI) of two layer Bayesian
Neural Networks (BNN). In this setting, the parameter weights for making predictions in practice are no longer fixed, but are sampled
 from a variational posterior, and the final result is the empirical average of the prediction of each sample. This variational posterior is obtained by maximizing an objective function, the Evidence
Lower Bound (ELBO) over a parameter space $\ThetaOne^N$ . It was empirically found that the maximization of the
"vanilla" ELBO function - based on the "vanilla" posterior distribution of a BNN-, can yield very poor inferences. To address this problem, a modification of this objective function was proposed, resulting from a decomposition into two terms of this function: one corresponding to the Kullback-Leibler (KL) divergence between the
variational density and the prior and the other to a marginal likelihood term.   Based on this decomposition, the modified version of
$\ELBOwN$, called partially tempered $\ELBOwN$, consists in multiplying the KL term by a
temperature parameter.

Although this change has been justified intuitively or by purely
statistical considerations, to our knowledge no formal results have
been derived. Our first contribution is to show that this procedure is
indeed mandatory in the overparameterized regime. More precisely, we
show that if the temperature is not scaled appropriately with respect
to the number of neurons and data points, one of the two terms becomes
dominant and therefore either the resulting posterior is too close to
the prior - underfitting- or overfitting occurs. Our second
contribution is to identify, similarly to risk minimization, a
limiting function for the partially tempered version of $\ELBOwN$ when
the number of neurons and data points approaches infinity and the
temperature parameter is appropriately chosen. Our conclusions are
twofold. First, we highlight the importance of the choice of prior and
variational family for the resulting variational posterior
distribution. Second, we show that performing VI for BNN in the
overparameterized regime is equivalent to risk over an extended space
of probability measures. In summary, then, the use of VI for BNN
enriches traditional NN models by producing predictions from a
hierarchical mixture distribution.  Ultimately, this shows that using
VI to train overparameterized BNN amounts to empirical risk
minimization and therefore should be interpreted with cautious for performing Bayesian uncertainty quantification, as our contributions imply that the resulting posterior
does not concentrate as the number of observations grows and is
therefore not related to the usual Bayesian posterior distribution of
the model.

This paper is organized as follows. \Cref{sec:background} introduces the background of VI on BNN. \Cref{sec:bnn_infinite} characterizes the inadequacy of these models in the limiting case of the mean field, when the data or prior variance do not scale, and identifies the well-posed regime.
In \Cref{sec:experiments}, some numerical experiments are presented to illustrate our claims.


\textbf{Notations.} We assume that $\rset^d$ is equipped with the Euclidean topology. For any set $\mse \subset \rset^d$, we denote by $\cP(\mse)$ the set of probability measures on $\mse$ equipped with the trace topology and the trace $\sigma$-field denoted by $\mce$. We define the weak topology on $\mcp(\mse)$ as the initial topology associated with $\nu \in \mcp(\mse) \mapsto \int f \rmd \nu$ for bounded and continuous map $f:\mse \to \rset$. For any $\nu,\mu\in \cP(\mse)$ we denote by $\nu \otimes \mu$ the product measure and define by induction $\nu^{\otimes N}=\nu^{\otimes N-1} \otimes \nu$.
 For $\sigma \in \R^{d}$, the diagonal matrix in $\R^{d\times d}$ with diagonal $\sigma^2 = [\sigma_1^2,\dots,\sigma_d^2]$ will be denoted $\diag(\sigma^2)$. For any measurable map $\scrT:\R^d \times \R^d$ and probability measure $\nu \in \cP(\R^d)$, we denote by $\scrT_{\#}\nu$ the pushforward measure of $\nu$ by $\scrT$. It is characterized by the transfer lemma, i.e. $\int F(y)\rmd \scrT_{\#}\nu(y)=\int F(\scrT(z))\rmd\nu(z)$ for any measurable and bounded function $F$. For $\mu,\pi \in \cP(\mse)$, the Kullback-Leibler divergence between a distribution $\mu$ and $\pi$ is defined by $\KL(\mu|\pi)=\int \log(\rmd\mu/\rmd\pi) d\mu$ where $\rmd\mu/\rmd\pi$ is the Radon-Nikodym derivative if $\mu$ is absolutely continuous with respect to $\pi$, and $\KL(\mu|\pi)=+\infty$ otherwise.


\section{Variational inference for BNN objective}\label{sec:background}

Consider a supervised setting where we have access to \iid~samples
$\{(x_i, y_i)\}_{i=1}^{p}$, from a distribution $\pi$ on $\X \times \Y \subset \rset^{\dx} \times \rset^{\dy}$, and aim at predicting $y$ given a new observation $x$.
In this paper, we focus on a fully connected NN with one hidden layer and $N$ neurons, and activation function $h : \rset^{\dx} \times \msx \to \rset$.
A common example of such a function is
\begin{equation} \label{eq:exampleh}
  h(b_j,x)= \upsigma(\ps{b,x}) \eqsp,
\end{equation}
for $b \in \R^{\dx}$ and $x \in\msx$, where $\upsigma:\R \to \R$ is
the Rectified Linear Unit $\upsigma(t) = \max(0,t)$ or sigmoid  function $\upsigma(t) = \rme^{t}/(1+\rme^{t})$,
 for $t \in\rset$.
 In addition, for $j\in \{1,\dots,N\}$, denote   by $b_j \in\rset^{\dx}$ and $a_j \in \rset^{\dy}$ the $j$-th weights of the hidden and output layers respectively, and set $w_j=(b_j,a_j) \in \rset^d$, $d = \dx +\dy$, $\bw=(w_j)_{j=1}^{N}$ all the weights of the NN under consideration.
With this notation, for each input $x\in \X$, the output prediction $f_{\bw} : \msx \to \rset^{\dy}$ of the neural network can be written as:
\begin{equation}\label{eq:deffw}
f_{\bw}(x)=N^{-1}\sum_{j=1}^{N}s(w_j, x),\text{ with } s(w_j, x) = a_j h(b_j,x) .
 \end{equation}
Given a loss function $\loss : \msy\times \msy \to \rset_+$, we use the prediction function $f_{\bw}$ to define the conditional likelihood
\begin{equation}\label{eq:deflikelihood}
  \likelihood(y|x,\bw)
\propto \exp(-\ell(f_{\bw}(x), y )) \eqsp,
\end{equation}
with respect to the Lebesgue measure on $\rset^{d \times N}$ denoted by $\Leb_{d \times N}$.
Then, choosing a prior pdf $\prior$ on $\bw$, the posterior pdf $\posterior$ of the weights 
is proportional to $\bw \mapsto \prior(\bw) \prod_{i=1}^p \likelihood(y_i |x_i,\bw)$.
We perform Bayesian inference using VI \citep{khan2021bayesian,blei2017variational,blundell2015weight,graves2011practical,khan2018fast}.
The general procedure is to consider a variational family of pdfs $\family_{\Theta} = \{\densityq_{\btheta} \, : \, \btheta \in \Theta \}$, for $\Theta \subset \rset^{\dtheta}$ and the Evidence Lower Bound (ELBO) defined for any  $\btheta \in\Theta$ by:
\begin{equation}
  \label{eq:elbo}
  \ELBO(\btheta) =  -\KL(\densityq_{\btheta} \, | \, \prior)+
   \sum_{i=1}^p \int_{\R^{N\times d}} \log \likelihood(y_i |x_i,\bw) q_{\btheta}(\bw) \rmd \Leb_{d \times N}(\bw) \eqsp.
 \end{equation}
It is known that maximizing $\ELBO$ is equivalent to minimizing $\btheta \mapsto \KL(\densityq_{\btheta} \, | \, \posterior)$. For this reason, VI consists in approximating the posterior distribution  by $\densityq_{\bthetas}$ with $\bthetas \in \argmax \ELBO$.
The first term in \eqref{eq:elbo} acts as a penalty term to control the deviation of $q_{\btheta^*}$ from the prior $\prior$, while the second term plays the role of empirical risk and promotes adaptation of the data.
In practice, however, it has been shown that the choice of the prior and the variational approximation $\ELBO$ is crucial for good performance. It was proposed by \cite{zhang2018noisy,khan2018fast,osawa2019practical,ashukha2020pitfalls} to weaken the regularization term KL and consider a partially tempered version of $\ELBO$, which for a cooling parameter $\etaN > 0$ is given by
\begin{equation}\label{eq:elboeta}
  \ELBOeta(\btheta) =  -\etaN\KL(\densityq_{\btheta} \, | \, \prior) 
 + \sum_{i=1}^p \int_{\R^{N\times d}} \log \likelihood(y_i |x_i,\bw) q_{\btheta}(\bw) \rmd \Leb_{d \times N}(\bw) \eqsp.
\end{equation}
It has been shown in \cite{wenzel2020good,wilson2020bayesian} that $\ELBOeta$ is the same as $\ELBO$ but considering instead of the common posterior $\posterior$, a partially tempered posterior $\posterior_{T_N} \propto \likelihood^{1/T_N} \prior$,  where the likelihood function is tempered for some temperature $T_N \geq 0$. The parameter $\etaN$ (or equivalently the temperature $T_N$) controls the tradeoff of the likelihood term with respect to the prior. Setting $\etaN < 1$ corresponds to a \emph{cold posterior}, where the likelihood term is strengthened so that the posterior is concentrated in regions of high likelihood. The case $\etaN=1$ corresponds to "plain" Bayesian inference, while $\etaN < 1$ corresponds to \textit{warm posterior} where the prior has a stronger influence on the posterior.

In a series of paper, \cite{grunwald2012safe,grunwald2017inconsistency,bhattacharya2019bayesian,heide2020safe,grunwald21PAC}  have shown - significantly extending earlier results of \cite{barron1991minimum,zhang2006information} -  that  partially tempered posteriors may have better statistical properties under model misspecification than the "plain" posterior as the number of data points goes to infinity (expressed in terms posterior contraction around the best approximation of the truth). These results have been derived for Generalized Linear Models and it is not clear how these results extend to BNN.
\cite{wilson2020bayesian} more informally argues that tempering is not inconsistent with Bayesian principles and that it may be particularly relevant in a parametric setting (where the model is defined by parameters), as opposed to Bayesian Nonparametric approaches - \eg, Gaussian processes. Namely, while in nonparametric approaches the model capacity is automatically scaled with the available data, this is not the case in parametric approaches, where the model capacity (which is determined by the number of neurons and the neural network architecture) is chosen by the user. Model misspecification is the rule in such case, as we show in \Cref{sec:bnn_infinite} for neural networks with a hidden layer. However, to the best of our knowledge, the temperature scaling with respect to the number of data points and network parameters has not been investigated theoretically, in particular in the context of BNN.

Other studies, \eg, \cite{farquhar2019radial}, noted that a potential cause of the predominance of the KL term in \eqref{eq:elbo} stems from the choice of the prior. Indeed, it has been noticed that the role of $\prior$ is important since it lead to very different inferences, see \cite{fortuin2021priors}. In particular, 
using priors on $\bw$ which factorize over the weights, \ie, of the form
\begin{equation}\label{eq:prioriid}
  \prior(\bw) = \prod_{j=1}^N \priorOne(w_j) \eqsp,
\end{equation}
do not yield optimal performance and as a result  \cite{tran2020all,fortuin2021bayesian,ober2021global,sun2018functional} have proposed the design of new priors which  introduce correlation amongst the weights and/or heavier tails than Gaussian ones.

In the present work, we take a novel approach to justify the use of $\ELBOeta$ based on the so-called  overparameterized regime and study the impact of the choice of the cooling parameter $\etaN$.
We assume that the prior factorizes over the neurons, \ie, the prior takes the form \eqref{eq:prioriid} and for each $\btheta = (\theta_1,\cdots,\theta_N) \in\ThetaOne^N$, $\densityq_{\btheta}(\bw) = \prod_{i=1}^N \densityqOne_{\theta_j}(w_j)$, where $\priorOne$ and $\{ \densityqOne_{\theta_j}\}_{j=1}^N$ are distributions over $ \ThetaOne \subset \rset^d$. In this case, $\Theta=\ThetaOne^N$ and the prior distribution for each neuron $\priorOne$ is the same.
Further,
we assume that for any $\theta \in \ThetaOne$, the variational distribution corresponding to $\densityqOne_{\theta}$ is the pushforward  of a reference probability measure with density  $\gammaOne$ by $\Tscr_{\theta}$, where $\{\Tscr_{\theta}\,:\, \theta \in \ThetaOne\}$ is a family of $\rmC^1$-diffeomorphisms on $\rset^{d}$, \ie,   $\densityqOne_{\theta}(w) = \gammaOne(\Tscr_{\theta}^{\inverse}(w)) \Jac_{\Tscr_{\theta}^{\inverse}}(w)$, denoting by $\Jac_{\Tscr_{\theta}^{\inverse}}$ the Jacobian determinant of $\Tscr_{\theta}^{\inverse}$. A common choice for $\Tscr_{\theta}$ is, setting $\theta = (\mu,\sigma) \in \rset^d \times (\rset_+^*)^d$,
\begin{equation}
  \label{eq:examplettheta}
\Tscr_{\theta}:z \mapsto \mu + \sigma \odot z \eqsp,
\end{equation}
where $\odot$ is the component wise product but of course much more sophisticated choices are possible.
Then, by~\eqref{eq:deffw}-\eqref{eq:deflikelihood} and a change of variable, the ELBO can be expressed as
\begin{equation}
 \label{eq:finitefunctional}
    \ELBOeta(\btheta) = -\sum_{i=1}^p \rmG^{N}_{\Theta}(\btheta;(x_i,y_i))
-\etaN    \sum_{j=1}^N \KL(\densityqOne_{\theta_j} | \priorOne)
   \end{equation}
with denoting the output of a neuron parametrized by $\theta\in \R^d$ for an input $x_i$ by
   \begin{equation}\label{eq:defphi}
     \phi(\theta,z,x_i) = s(\Tscr_{\theta}(z),x_i) \eqsp,
   \end{equation}
   and
   $\bz = (z_1,\ldots,z_N)\in \R^{d\times N}$,
   \begin{equation}
     \label{eq:defrmGN}
   \rmG^{N}_{\Theta}(\btheta;(x,y))=   \int \loss\left( y,  \sum_{j=1}^N \frac{\phi(\theta_j,z_j,x)}{N}
\right)
\gammaOne^{\otimes N}(\rmd \bz) \eqsp.
\end{equation}
Although the VI framework we are considering may seem overly simplistic in light of the above, it is the one most commonly used in practice, and therefore it is still very important to obtain useful guidelines for implementation in order to optimize its performance. Moreover, it is a first step before considering other VI methods with more complex priors and/or variational families.
The expression of $\ELBOeta$ shows that the parameter $\etaN$ must be chosen to balance the two terms in \eqref{eq:elboeta} to obtain a well-posed VI functional as $N,p \to \plusinfty$ and a variational posterior $\densityq_{\bthetas}$ different from the prior. Without this parameter, optimizing the $\ELBO$ \eqref{eq:elbo} leads to the collapse of the variational posterior to the prior, as shown in the following proposition. 

\begin{proposition}\label{prop:posterior_collapse}
 Assume that $\family_{\Theta}$ is a family of Gaussians with diagonal covariance matrices, that $\prior \in \family_{\Theta}$ and that $\X$ is compact.   
Let $\btheta^{*,N} = \argmax_{\btheta \in \Theta}\ELBO(\btheta)$.  Assume also that $l$ is the square loss or cross-entropy, and that $\upsigma$ is Lipschitz. 
Then, $\KL(\densityq_{\btheta^{*,N}}, \prior)\to 0$ as $N\to \infty$.
\end{proposition}
This result and its proof, that can be found in \Cref{sec:proof_posterior_collapse}, 
are inspired from \citet[Theorem 1,2]{coker2021wide} that show that the moments of the predictive posterior collapse to the ones of the prior and that the KL converges to 0 as $N\to \infty$, when $l$ is the square loss and $\upsigma$ is odd. \Cref{prop:posterior_collapse} states an analog result, but which holds for additional losses (i.e., also cross-entropy) and more general activation functions (e.g., non odd ones as RelU). This is partly due to our different scaling of the output of the neural network, in $O(1/N)$ (see \eqref{eq:deffw}) that differs from theirs (in $O(1/\sqrt{N})$). 
The result of \Cref{prop:posterior_collapse} highlights that optimizing $\ELBO$ becomes ill-posed as $N\to \infty$. 
This suggests that 
the optimal variational posterior tends to ignore the data fitting term in \eqref{eq:finitefunctional}, 
and that $\etaN$ must be chosen to rebalance $\ELBO$. 
In the next section, we provide a theoretical framework supporting this informal discussion and then present our main results regarding the choice of $\etaN$.


\section{Identifying well-posed regimes for the ELBO 
with product priors}\label{sec:bnn_infinite}

We follow the approach outlined in \cite{chizat2018global,rotskoff2019global,mei2018mean} for ERM. We first generalize the definition of $\ELBOeta$ defined in \eqref{eq:finitefunctional} over $\ThetaOne^N$ to probability measures $\nu$ on $\ThetaOne$. Indeed, the following result states that $\NELBO$ can be expressed as a functional of the empirical measure over the weights $\nu_N^{\btheta}$ defined for each $\btheta = (\theta_1,\ldots,\theta_N) \in \ThetaOne^N$ by
\begin{equation}\label{eq:empiricalmeasure}
  \nu_N^{\btheta} = N^{-1}\sum_{i=1}^N \updelta_{\theta_i} \eqsp,
\end{equation}
where $\updelta_{\theta}$ is the Dirac mass at $\theta \in \ThetaOne$. Define $\mcp_{N}(\ThetaOne)$ the subset of $\mcp(\ThetaOne)$ which can be written as \eqref{eq:empiricalmeasure} for some $\btheta \in\ThetaOne^N$.
\begin{proposition}
  \label{propo:1}
For any $N \in\nsets$, there exists a function $\funFN$ defined over $\mcp_{N}(\ThetaOne)$ valued in $\rset\cup\{\plusinfty\}$ such that $\ELBOeta(\btheta) = \funFN(\nu_N^{\btheta})$ for any $\btheta \in\ThetaOne^N$.
\end{proposition}

\begin{proof}
Denote by $\scrS_N$ the set of permutations $\tau : \{1,\ldots,N\} \to \{1,\ldots,N\}$ and for any $\btheta = (\theta_1,\ldots,\theta_N)\in\Theta$, $\tau \in\scrS_N$, $\btheta^{\tau} =  (\theta_{\tau(1)},\ldots,\theta_{\tau(N)})$. Note that for any $\tau \in \scrS_N$, $\ELBOeta(\btheta) = \ELBOeta(\btheta^{\tau})$. The proof is then completed upon using that  $\btheta \mapsto \nu_N^{\btheta}$ is a bijection from $\Xi^N/\sim$ to $ \mcp_{N}(\ThetaOne)$, where $\sim$ is the equivalence relation defined by $\btheta \sim \btheta'$ if  $\exists\tau \in \scrS_N$ s.t. $\btheta' = \btheta^{\tau}$.
\end{proof}

\Cref{propo:1} is a first step toward our results. However, its main caveat is that  $\funFN$ cannot be non-trivially extended to a functional defined for a general probability measure on $\ThetaOne$.  However, in our next result, we show that, when restricted to empirical probabilities, it is a perturbation, as $N \to \plusinfty$, of the functional $\tfunFN$ defined over all $\mcp(\ThetaOne)$ by
\begin{equation}
 \label{eq:trmFN}
 \tfunFN(\nu) = -\sum_{i=1}^p \trmGN(\nu;(x_i,y_i))
-\etaN  N \int \KL(\densityqOne_{\theta} | \priorOne) \rmd \nu(\theta) \eqsp, 
   \end{equation}
   where
   \begin{equation}
     \label{eq:trmGN}
   \trmGN(\nu;(x,y))=   \loss\left( y,  \iint\phi(\theta,z,x) \rmd\nu(\theta) \rmd \gamma(z) \right) \eqsp,
\end{equation}
and $\phi$ is given by \eqref{eq:defphi}. We now define for any
$\theta\in\ThetaOne$ and $x\in\msx$,
$\tphi(\theta,x) = \int \phi(\theta,z,x) \rmd \gamma(z)$. Consider the
following assumption:
\begin{assumption}
  \label{ass:unifbound}
  \begin{enumerate}[label=(\roman*),wide, labelwidth=!, labelindent=0pt]
  \item   \label{ass:unifboundi} There exists $\Liploss > 0$ such that for any $y \in\msy$,  the function $\tilde{y} \mapsto
    \loss(y,\tilde{y})$ is $\Liploss$-smooth: for any $\tilde{y}_1,\tilde{y}_2 \in\msy$,
    \begin{equation}
      \label{eq:10}
      \norm{\nabla_{\tilde{y}}    \loss(y,\tilde{y}_1)-\nabla_{\tilde{y}}    \loss(y,\tilde{y}_2)} \leq \Liploss \norm{\tilde{y}_1-\tilde{y}_2} \eqsp.
    \end{equation}
\item \label{ass:unifboundii} There exists $\constphi \geq 0$,  such that for any $\theta \in \ThetaOne$ and $x \in\msx$,
  \begin{equation}
    \label{eq:varphibound}
    \int \norm{\phi(\theta,z,x) - \tphi(\theta,x)}^2\rmd \gamma (z) \leq \constphi \eqsp.
  \end{equation}
\end{enumerate}
\end{assumption}
Note that \Cref{ass:unifbound}-\ref{ass:unifboundi} is satisfied for the quadratic or logistic loss if $\msy$ is bounded. We give practical conditions on the activation function $\upsigma$, the prior $\priorOne$ and the set $\ThetaOne$ to ensure that \Cref{ass:unifbound}-\ref{ass:unifboundii} holds in the case where $\scrT_{\theta}$ is supposed to be of the form \eqref{eq:examplettheta} for any $\theta\in\ThetaOne$, later in this section after stating our general results.
\begin{theorem}
  \label{theo:unifbound}
  Assume \Cref{ass:unifbound}. Then, there exists $C \geq 0$ such that for any $N,p \in \nsets$, $\{(x_i,y_i)\}_{i=1}^p \in (\msx \times \msy)^p$,  $\btheta \in \ThetaOne^N$ and $\etaN >0$,
  \begin{equation}
    \label{eq:2}
    \abs{    \ELBOeta(\btheta) -  \tfunFN(\nu_N^{\btheta})} \leq C p/N \eqsp,
\end{equation}
where $\nu_N^{\btheta}$ is defined in \eqref{eq:empiricalmeasure}.
\end{theorem}
\begin{proof}
Using that for any $y \in\msy$,   the function 
$\tilde{y} \mapsto \loss(y,\tilde{y})$ is $\Liploss$-smooth, we get by \citep[Lemma 1.2.3]{nesterov:2004}, \Cref{propo:1} and the definitions \eqref{eq:finitefunctional}-\eqref{eq:defrmGN}-\eqref{eq:trmFN}-\eqref{eq:trmGN},
  \begin{align}
    \label{eq:4}
    &    \abs{    \funFN(\nu_N^{\btheta}) -  \tfunFN(\nu_N^{\btheta})} \leq \frac{\Liploss}{2 N^{2}}\sum_{i=1}^p \int_{}\, \norm{\sum_{j=1}^N \phi(\theta_j,z_j,x_i) - \tphi(\theta_j,x_i)}^2 \rmd \gamma^{\otimes N}(\bz)\\
    & \leq \frac{\Liploss}{2 N^2}\sum_{i=1}^p \sum_{j=1}^N \int_{} \,\norm{ \phi(\theta_j,z_j,x_i) - \tphi(\theta_j,x_i)  }^2  \rmd \gamma(z) \eqsp.
  \end{align}
The proof follows from \Cref{ass:unifbound}-\ref{ass:unifboundii}.
\end{proof}
We can also show that minimization of $\funFN$ over $\mcp_{N}(\ThetaOne^N)$ provides a good approximation for the minimization problem corresponding to $\tfunFN$ for sufficiently large $N$.
\begin{theorem}
  \label{theo:2}
  Assume \Cref{ass:unifbound} and that there exists $\nus \in \mcp(\ThetaOne)$ such that $\nus \in \argmax_{\mcp(\ThetaOne)} \tfunFN$. Suppose in addition that there exists $\constphinus \geq 0$ such that for any $x \in\msx$,
  \begin{equation}
    \label{eq:condition_theo_2}
    \int \norm{\tphi(\theta,x)-  \int \tphi(\theta',x) \rmd \nus(\theta')}^2 \rmd \nus(\theta)
    \leq \constphinus  \eqsp.
  \end{equation}
  Then, there exists $C \geq 0$ such that for any $N,p \in \nsets$, $\{(x_i,y_i)\}_{i=1}^p \in (\msx \times \msy)^p$ and $\etaN >0$,
  \begin{equation}
    \label{eq:2}
    \abs{    \sup_{\btheta \in \ThetaOne^N}\ELBOeta(\btheta) -  \sup_{\nu \in \mcp(\ThetaOne)}\tfunFN(\nu)} \leq C p/N \eqsp.
\end{equation}
\end{theorem}
\begin{proof}
Using \Cref{theo:unifbound}, we easily get that for any $\btheta \in\ThetaOne^N$,
  \begin{equation}
    \label{eq:5}
    \ELBOeta(\btheta)  \leq  \tfunFN(\nu_N^{\btheta}) + Cp/N \leq \sup_\nu \tfunFN(\nu) +C p/N \eqsp,
 \end{equation}
 for some $C \geq 0$ independent of $\{(x_i,y_i)\}_{i=1}^p \in (\msx \times \msy)^p$ and $\etaN >0$.
 On the other hand, we have using $\nus$ is a maximizer of $\tfunFN$,
\begin{multline}
\label{eq:unfimin1}
\sup_{\btheta \in \ThetaOne^N}   \ELBOeta(\btheta)  \geq  \\ 
\sup_\nu \tfunFN(\nu) - \int \abs{ \ELBOeta(\btheta) - \tfunFN(\nu_N^{\btheta}) }\rmd \nus^{\otimes N}(\btheta) -  \int \abs{ \tfunFN(\nu_N^{\btheta}) - \tfunFN(\nus)} \rmd \nus^{\otimes N}(\btheta) \eqsp.
\end{multline}
Using  \Cref{ass:unifbound}, for any $y \in\msy$,  
$\tilde{y} \mapsto \loss(y,\tilde{y})$ is $\Liploss$-smooth, we get by \citep[Lemma 1.2.3]{nesterov:2004},
 \begin{multline}
\label{eq:unfimin2}
 \hspace{-0.5cm} \int \abs{ \tfunFN(\nu_N^{\btheta}) - \tfunFN(\nus)} \rmd \nus^{\otimes N}(\btheta)  \leq \Liploss \sum_{i=1}^p \int \norm{\frac{1}{N} \sum_{j=1}^N \tphi(\theta_j,x_i) - \int \tphi(\theta',x_i)\rmd \nus(\theta')}^2 \rmd \nus^{\otimes N}(\btheta)\\
  \leq \Liploss p \constphinus/N \eqsp.
\end{multline}

Combining \eqref{eq:unfimin1}, \eqref{eq:unfimin2} and \Cref{theo:unifbound} completes the proof.
\end{proof}
We now set $\etaN = \etawN p/N$ with $\etawN >0$. With this particular choice, $\tfunFN$ depends only on the number of observations $p$ but no longer on the number of neurons $N$. We denote, for that particular choice of $\etaN$,
\begin{equation}
\funFinftyp (\nu)
= p^{-1} \tfunFN(\nu) 
= -\frac{1}{p}\sum_{i=1}^p \trmGN(\nu;(x_i,y_i))
-\etawN  \int \KL(\densityqOne_{\theta} | \priorOne) \rmd \nu(\theta) \eqsp.
\end{equation}
In our next result, we show that with high probability, $\funFinftyp(\nu)$ provides a good approximation of the function
\begin{equation}
  \label{eq:6}
  \risk(\nu) =- \int \trmGN(\nu;(x,y)) \rmd \pixy(x,y) 
  -\etawN \int \KL(\densityqOne_{\theta} | \priorOne) \rmd \nu(\theta) \eqsp,
\end{equation}
where $\trmGN$ is defined by \eqref{eq:trmGN}. The first term is an integrated log-likelihood with respect to the joint distribution of observations. The second term is a form of penalization of the complexity of  $\nu$. It is interesting to note that this penalty term is defined from the prior $\prior$ and the variational family $\family_{\Theta}$.
\begin{proposition}
  \label{propo:2}
  Assume \Cref{ass:unifbound} and that there exists $M_G>0$, s.t. for any $\nu \in \mcp(\ThetaOne)$, $0\leq \trmGN(\nu;(x,y)) \leq M_G$, for $\pixy$-almost all $(x,y) \in\msx \times \msy$. Suppose in addition that $\{(x_i,y_i)\}_{i=1}^p$ are \iid~with distribution $\pixy$. Then, for any $\nu \in \mcp(\ThetaOne)$ and $\delta >0$,
  with probability $1-\delta$ at least, it holds
  \begin{equation}
    \label{eq:3}
\abs{ \funFinftyp(\nu) -     \risk(\nu)} \leq M_G \sqrt{\log(\delta/2)/(2p)} \eqsp.
  \end{equation}
\end{proposition}

 The proof follows from applying Hoeffding's inequality on the bounded i.i.d. variables $\trmGN(\nu;(x_i,y_i))$ for $i=1,\dots,p$. 

 It is worth noting that this limiting risk is  similar to the one 
 obtained in the analysis of the limiting behavior of gradient descent type algorithms for one or two hidden layers in the overparameterized regime, by \cite{chizat2018global,rotskoff2019global,mei2018mean,tzen2019neural,de2020quantitative} - see \eqref{eq:7}. Moreover, the maximization of the ELBO using gradient descent  can be viewed as a temporal and spatial discretization of the Wasserstein gradient flow of the limiting function \eqref{eq:6}.

For simplicity, we illustrate our result for  a mean-field variational family associated with the family of $\rmC^1$-diffeomorphisms on $\rset^{d}$, $\{\Tscr_{\theta}\,:\, \theta \in \ThetaOne\}$ given in \eqref{eq:examplettheta} for $\ThetaOne \subset \rset^d \times (\rset_+^*)^d$. Consider the following assumption.
\begin{assumption}
  \label{ass:gamma_h}
  \begin{enumerate}[label=(\roman*),wide, labelwidth=!, labelindent=0pt]
  \item  The subset $\ThetaOne$ is a compact set of $\rset^d \times (\rset_+^*)^d$, and $\X,\Y$ are compact sets of $\R^{\dx},\R^{\dy}$. 
  \item The probability measure $\gamma$ satisfies $\int \norm{z}^4 \rmd \gamma(z) < \plusinfty$.
  \item  For any $x \in \msx$, there exists $\Liph \geq 0$ such that the function $b \mapsto h(b,x)$ is $\Liph$-Lipschitz on $\rset^{\dx}$ and $\sup_{x\in\msx\, , \, b \in \rset^{\dx}} \abs{h(b,x)}/(1+\norm{b}) < \plusinfty$.  
  \item The prior density $\priorOne$ is positive on $\rset^d$ and satisfies $\theta \mapsto \KL(\densityqOne_{\theta}|\priorOne)$ is continuous on $\ThetaOne$.
  \end{enumerate}
\end{assumption}
Note that the condition that for any $x \in \msx$, the condition $b \mapsto h(b,x)$ is $\Liph$-Lipschitz is automatically satisfied for $h$ of the form \eqref{eq:exampleh} with $\upsigma$ the RELU or sigmoid function if $\msx$ is bounded. Also, we verify in the next proposition that $\theta \mapsto \KL(\densityqOne_{\theta}|\priorOne)$ is  continuous if $\priorOne$ and $\gamma$ are non-degenerate Gaussian distributions.

\begin{proposition}
  \label{propo:verify_A1}
  Assume \Cref{ass:unifbound}-\ref{ass:unifboundi} and \Cref{ass:gamma_h}. Then \Cref{ass:unifbound}-\ref{ass:unifboundii} and the conditions of \Cref{theo:2} hold.
\end{proposition}
\begin{proof} We first prove \Cref{ass:unifbound}-\ref{ass:unifboundii}. Recall, that for $\theta,z,x\in \ThetaOne\times\R^d\times\X$,
  $\phi(\theta, z, x) = s(\scrT_{\theta}( z), x)$ where $\scrT_{\theta}( z) = \mu + \sigma\odot z$. Therefore, by \eqref{eq:deffw}, decomposing each weight as $w =(a,b)$ where $a$ is the output weight and $b$ is the hidden weight,
  $  \phi(\theta, z, x) = a h(b,x) \eqsp,$
  with $ a = \mu_a + \sigma_a \odot z_a$ and $b = \mu_b +\sigma_b \odot z_b$, $\theta = (\theta_a,\theta_b)$, $\theta_a = (\mu_a,\sigma_a) \in \rset^{\dy} \times (\rset_+^*)^{\dy}$ and $\theta_b = (\mu_b,\sigma_b) \in \rset^{\dx} \times (\rset_+^*)^{\dx}$.
  Hence,
  \begin{equation}
    \label{eqbounda}
    \|a\|^2\le 2\|\mu_a\|^2+2\|\sigma_a\|^2\|z_a\|^2 \le 2\|\theta\|^2(1+\|z_a\|^2) \eqsp,\quad
    \norm{b}^2 \leq 2\|\theta\|^2(1+\|z_b\|^2) \eqsp.
  \end{equation}
  Also, by~\Cref{ass:gamma_h}, there exist $C_0,C_1 \geq 0$ such that for any $x,b$, $|h(b,x)|\le C_0 + C_1\|b\|$. Hence, 
  we have for any $\theta \in\ThetaOne$, $z \in \rset^d$ and $x \in\msx$,
\begin{align}
&\|\phi(\theta,z,x)\|^2\le \|a\|^2(C_0+C_1\|b\|)^2\\
  &
  \label{eq:bound_phi_z_x}\le 2\norm{\theta}^2(1+\norm{z_a}^2)[C_0+2 C_1\norm{\theta}(1+\norm{z_b})^{\frac{1}{2}}]^2 
\le C_3 (1+\norm{z}^4)(1+\norm{\theta}^2) \eqsp,
\end{align}
for some constant $C_3 >0$. As $\ThetaOne$ is compact and $\int \|z\|^4\rmd\gamma(z)<+\infty$,  it follows that \Cref{ass:unifbound}-\ref{ass:unifboundii} holds.
We now show that $\argmax_{\mcp(\ThetaOne)} \tfunFN \neq \emptyset$.
By  \Cref{eq:prop_tphi} in the supplement, $\tphi$ is bounded and for any $x \in \msx$, $\theta \mapsto \tphi(\theta,x)$ is continuous. Using that under \Cref{ass:gamma_h} for any $y \in\msy$, $\tilde{y} \mapsto \loss(y,\tilde{y})$ is continuous, it follows that $\nu \mapsto    \trmGN(\nu;(x,y))$ is continuous for the weak topology on $\mcp(\ThetaOne)$ for any $(x,y)\in\msx\times \msy$. In addition, since $\theta\mapsto \KL(\densityqOne_{\theta} | \prior)$ is continuous, we get since $\ThetaOne$ is compact that  $\nu \mapsto\int \KL(\densityqOne_{\theta} | \priorOne) \rmd \nu(\theta)$ is continuous for the weak topology.  It follows that $\nu \mapsto \tfunFN(\nu)$ is continuous for the weak topology. Using  $\ThetaOne$ is compact, $\cP(\ThetaOne)$ is compact for the weak topology by \citet[Theorem 5.1.3]{ambrosio2008gradient}, and it follows that $\argmax_{\mcp(\ThetaOne)} \tfunFN \neq \emptyset$. 
The last condition \eqref{eq:condition_theo_2} of  \Cref{theo:2} easily follows from \Cref{eq:prop_tphi}.
\end{proof}


\section{Experiments}\label{sec:experiments}

In this section we illustrate our  findings and their practical implications for image classification on standard datasets (MNIST, CIFAR-10). The reader may refer to \Cref{sec:additional_experiments} for additional experiments, including on regression tasks, that highlight the importance of rescaling the ELBO. Our code is available at \url{https://github.com/THuix/bnn_mfvi}. In this section, we illustrate the influence of the parameter $\tau$ through different metrics. 



\textbf{Evaluation.}  Let $\cD=(x_i,y_i)_{i=1}^p$  be a dataset, where $y_i=c\in \{1,\dots,n_l\}$ is a discrete class label. For an input $x\in \X$, the predictive probability of a class $c$ by a neural network with weights $\bw$ is  defined by $\Psi_c(f_{\bw}(x))$, where $\Psi_c(f_{\bw}(x))$ denotes the $c$-th component of the softmax function applied to the output $f_{\bw}(x)\in \R^{n_l}$ of the neural network.  
The cross entropy loss writes
$\lossCE(y,f_{\bw}(x)) =-\sum_{c=1}^{n_l} \tilde{y}_c \log(\Psi_c(f_{\bw}(x))) 
$, where $\tilde{y}_c$ denotes the $c$-th coordinate of a one-hot representation of the label $y$ and the Negative Log Likelihood (NLL)  $\sum_{i=1}^p \int_{\R^{N\times d}} \lossCE(y_i,f_{\bw}(x_i)) q_{\btheta}(\bw) \rmd \Leb_{d \times N}(\bw)$.
The calibration performance of the model can be estimated by the Expected Calibration Error (ECE) \cite{naeini2015obtaining}, see also \Cref{sec:def_ece}. We recall that model is calibrated if
the predictive posterior 
is the true probability for each class $c\in \{1,\dots,n_l\}$. However, since these probabilities are unknown, they have to be estimated, e.g. through ECE. 
As the NLL, ECE penalizes low probabilities assigned to correct predictions and high probabilities assigned to wrong ones; but these evalutation metrics are not strictly equivalent.
To make our prediction, for $x\in\msx$, we use  the posterior predictive distribution defined for a class $c$ as $ \int \Psi_c(f_{\bw}(x))\densityq_{\bthetas}(\bw) \rmd \Leb_{d \times N}(\bw)$ with $\bthetas$ obtained by minimization of $\ELBOeta$ by Bayes by Backprop. This integral is estimated by an empirical version
\begin{equation}\label{eq:approximatepreddistribution}
\int \Psi_c(f_{\bw}(x))\densityq_{\bthetas}(\bw)\rmd \Leb_{d \times N}(\bw) \approx \frac{1}{m} \sum_{l=1}^m \Psi_c(f_{\bw_l}(x)) \eqsp,
  \end{equation}
  where for $l=1,\dots,m$, $\bw_l$ are \iid~samples from $\densityq_{\bthetas}$. All the evaluation metrics mentioned above (NLL, ECE), as well as the accuracy are estimated using the same procedure.
We will present our results on the MNIST dataset (where $p=6.10^4$) and the CIFAR-10 dataset (where $p=5.10^4$) \cite{krizhevsky2009learning}.

\textbf{Setup.} 
We use a Linear BNN on MNIST, and ResNet20 architecture \cite{he2016deep} on CIFAR-10 \cite{simonyan2014very}. For CIFAR-10, we use the standard data augmentation techniques, see \cite{khan2018fast}. 
For each neuron, we use a centered Gaussian prior with variance $\nicefrac{1}{5}$, following  \cite{osawa2019practical}. 
We train each BNN by Bayes by Backprop \cite{blundell2015weight} with the reparametrization trick (see  \Cref{sec:bayes_by_backprop}) and using batch normalization \cite{sergey2021batch}.  

\begin{figure}[t]
   \centering
\begin{tabularx}{\columnwidth}{cccc}
\tiny Accuracy & \tiny NLL &\tiny ECE & \tiny Confidence\\
	\includegraphics[width=.22\columnwidth]{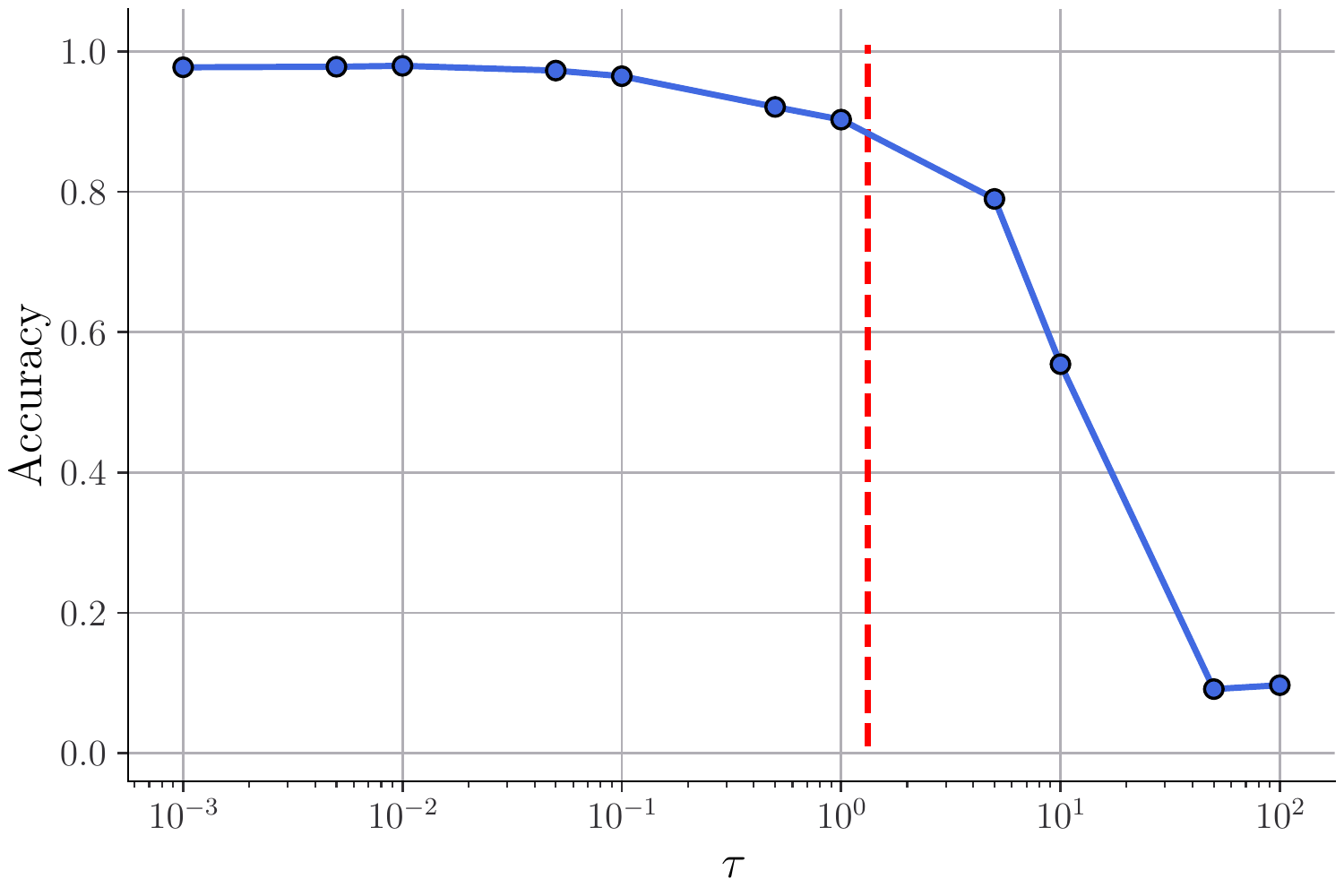}&
	\includegraphics[width=.22\columnwidth]{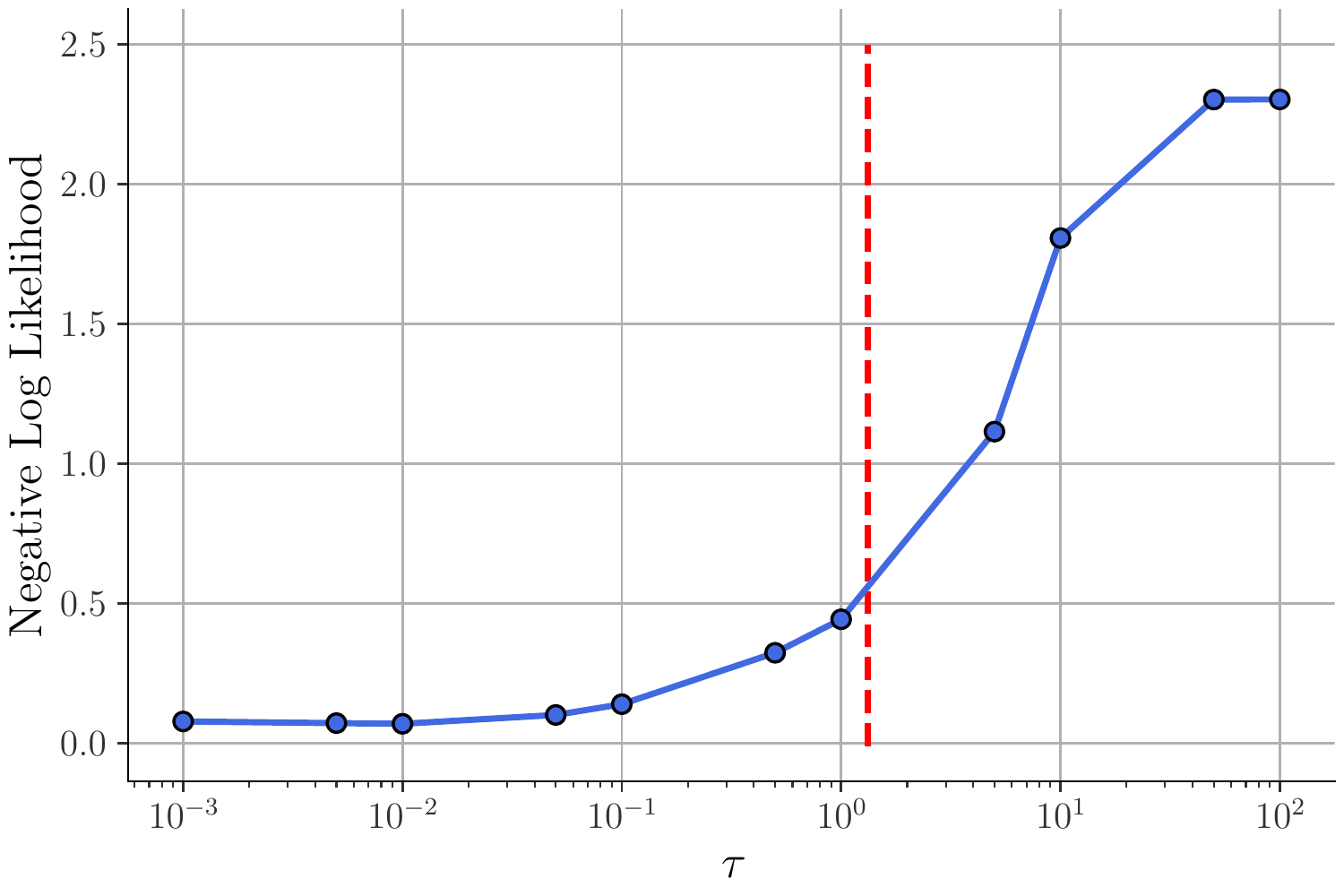}&
  \includegraphics[width=0.22\columnwidth]{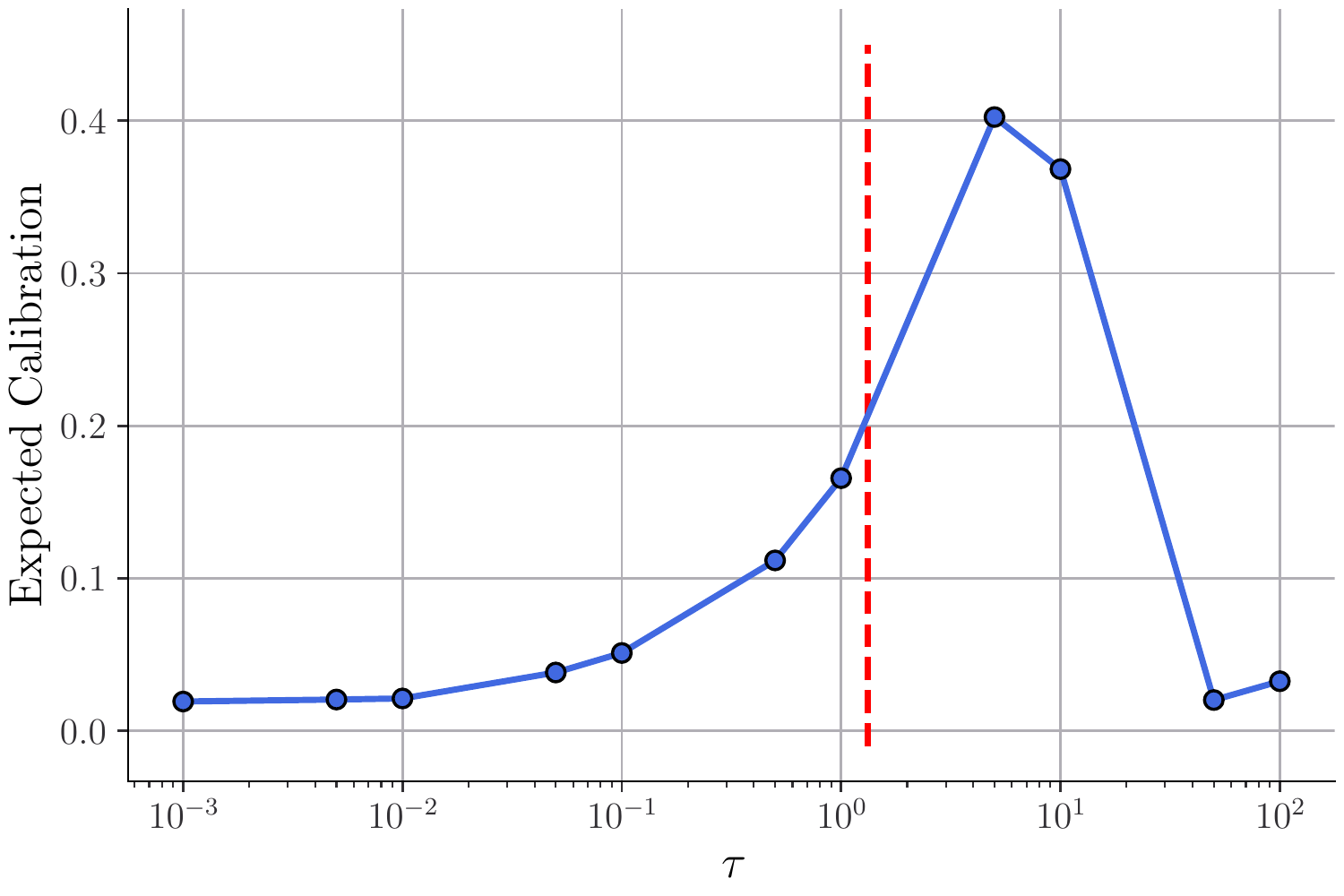} &\includegraphics[width=.22\columnwidth]{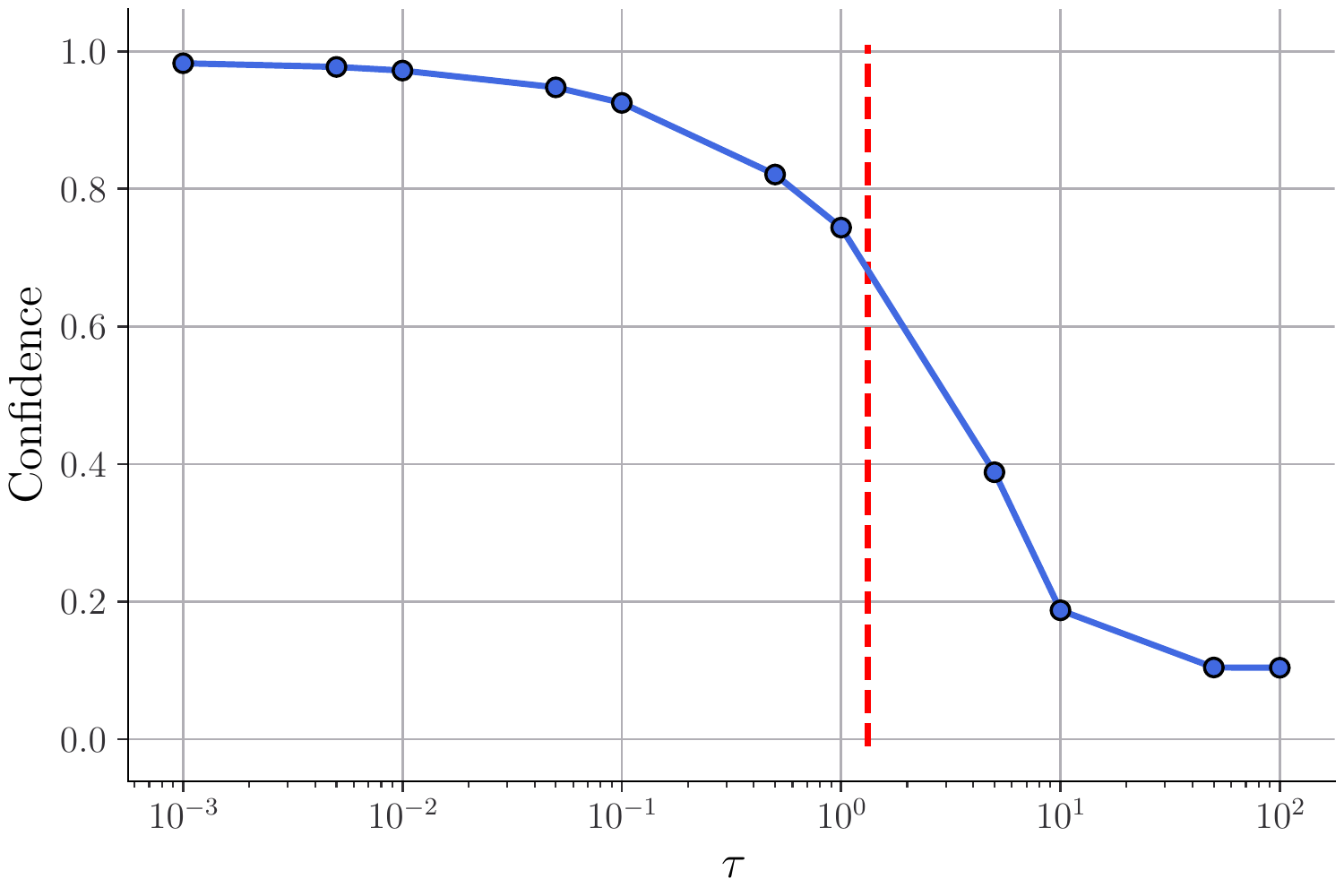}
  \\
\end{tabularx}
    \caption{Effect of the temperature for a Linear BNN trained on MNIST. No cooling $\etaN=1$ is indicated by a red line.}
    \label{fig:linear}
\end{figure}

\begin{figure}[t]
   \centering
\begin{tabularx}{\columnwidth}{cccc}
\tiny Accuracy & \tiny NLL &	\tiny ECE & \tiny Confidence \\
	\includegraphics[width=.22\columnwidth]{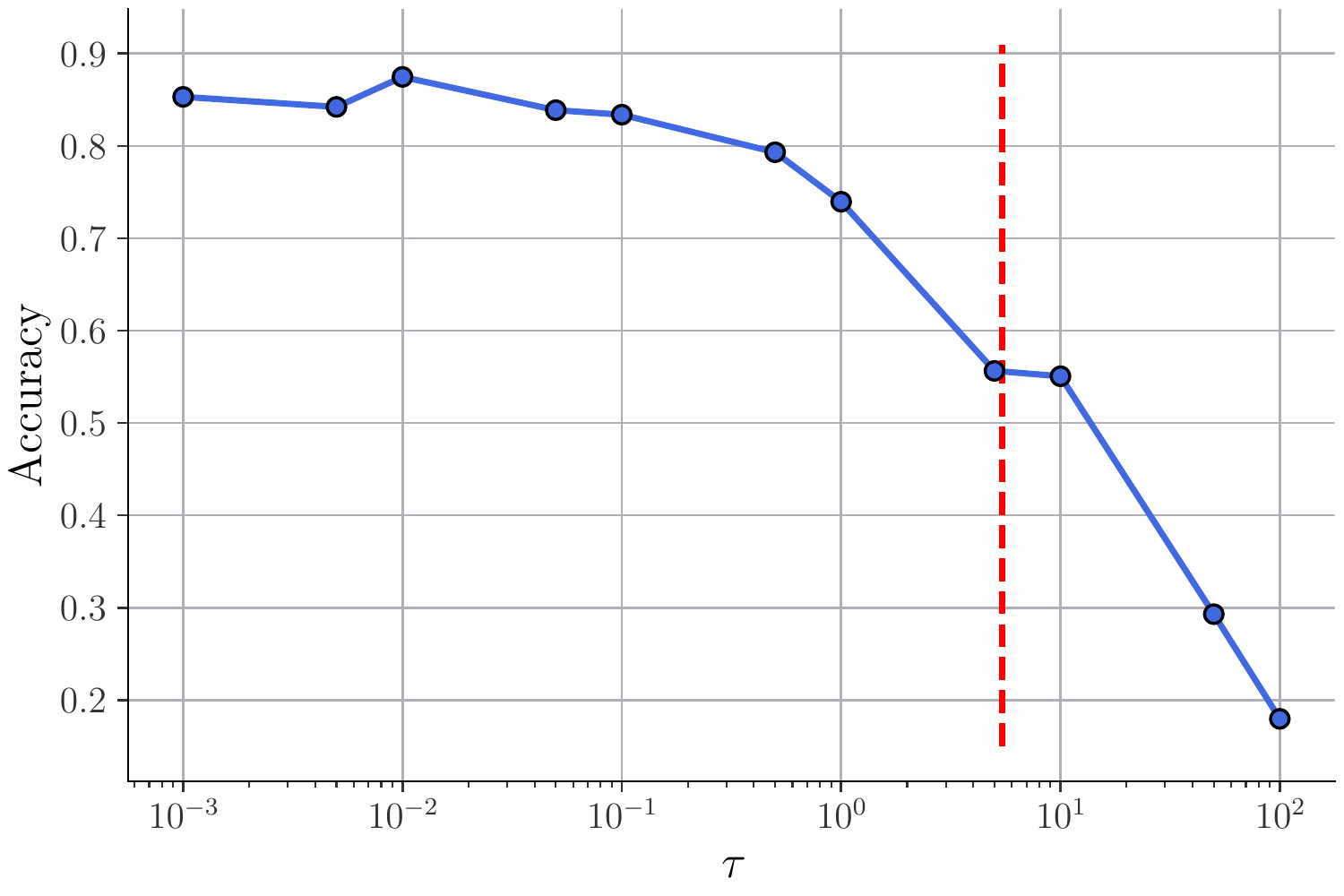}&
	\includegraphics[width=.22\columnwidth]{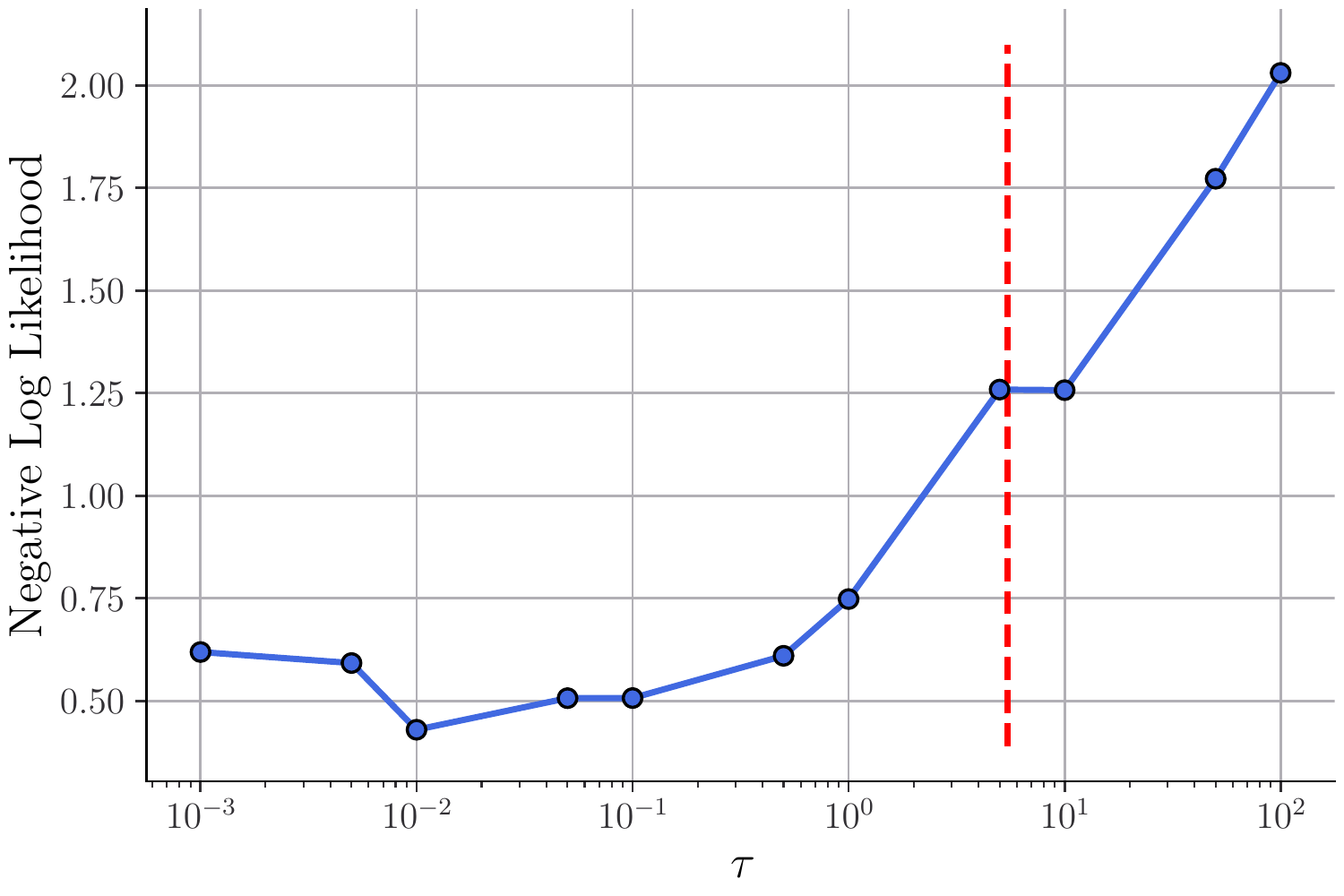}&
  \includegraphics[width=0.22\columnwidth]{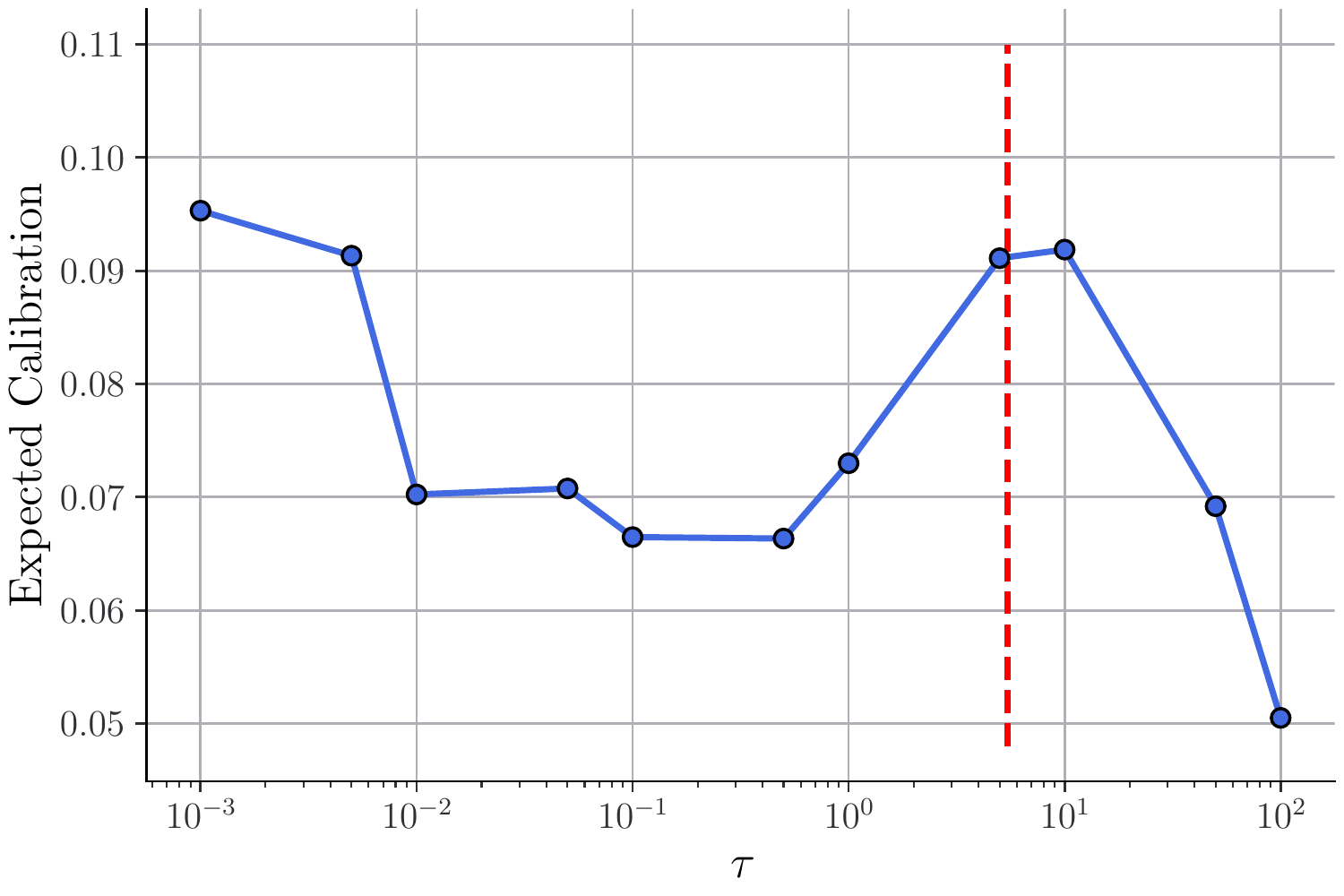} &\includegraphics[width=.22\columnwidth]{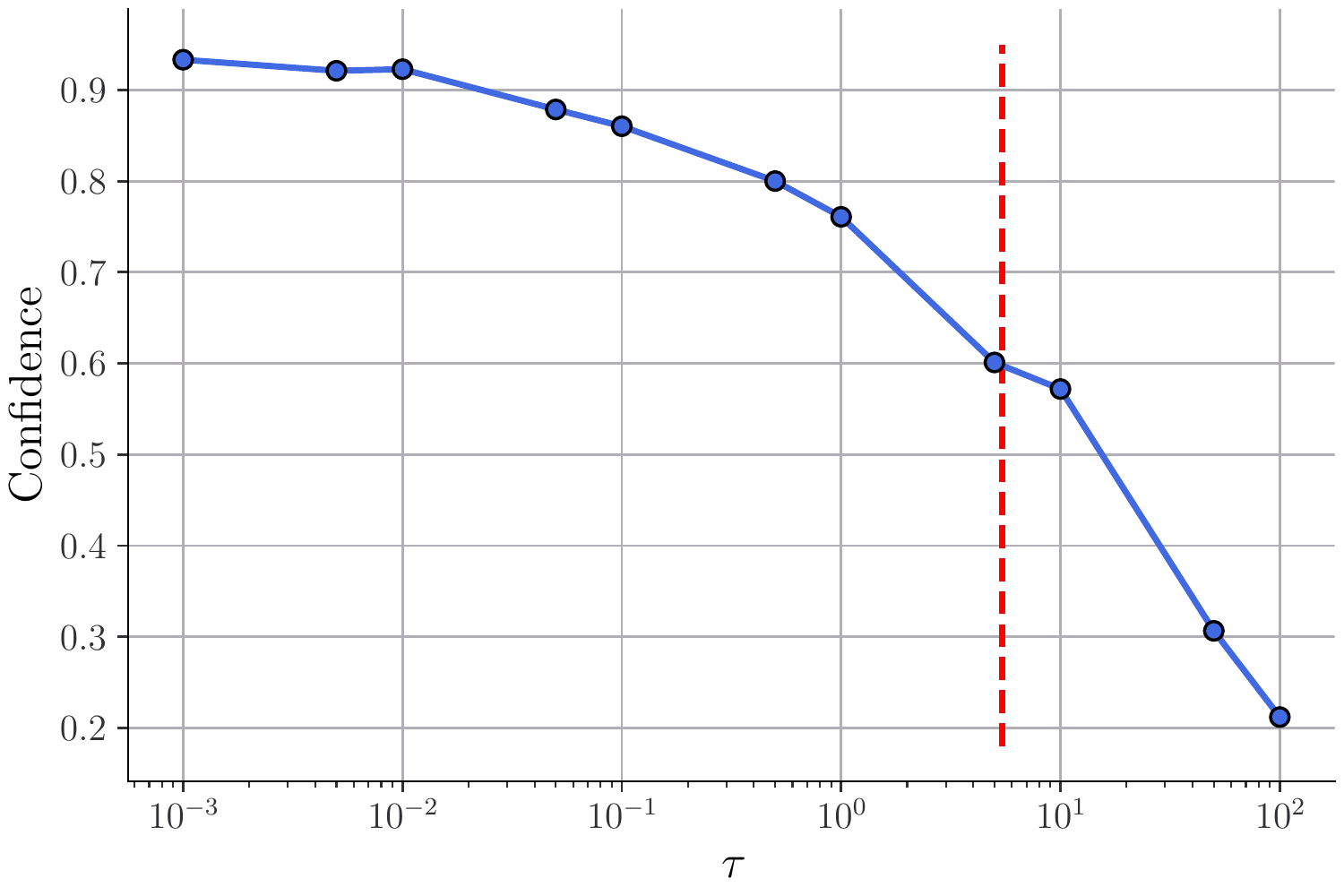}
  \\
\end{tabularx}
    \caption{Effect of the temperature for a Resnet20 trained on CIFAR-10.    No cooling $\etaN=1$ is indicated by a red line.}
    \label{fig:resnet}
\end{figure}

\Cref{fig:linear}, \ref{fig:resnet} illustrate the performance of the
different models and data sets for different values of $\tau$.  We
evaluate the models on the test set in terms of their accuracy, NLL,
ECE, and average confidence over the test set. In all experiments, we
take $m=50$ in \eqref{eq:approximatepreddistribution} to approximate a
BNN prediction and average our results over $5$ experiments for each $\tau$.
It is worth noting that for a large $\tau$, the
accuracy decreases while the NLL increases.  This
is hardly a surprise, since the KL regularization forces the VI
posterior to stay close to the prior distribution, resulting in
underfitting. At the same time, the ECE value is low because of the poor
confidence in the model, which is reflected in the accuracy.  For
small values of $\tau$, the data fitting term is privileged, so the
accuracy of the model is high, while the NLL is
low. At the same time, the confidence in the model is very high, resulting in a
low ECE.  For intermediate values of  $\tau$, the accuracy of the
models starts to decrease, but slower than the confidence in
the model, which explains an increase in ECE. We also illustrate the
different regimes for the parameter $\tau$ with additional experiments
in \Cref{sec:additional_experiments}, including analysis of the weights distribution and
out-of-distribution detection.

\section{Conclusion}

In this work, we studied BNN trained with mean-field VI in the overparameterized regime. We have highlighted both theoretically and numerically that the partially tempered $\ELBOeta$ advocated for VI for BNN effectively addresses the potential imbalance between the data fitting and KL terms. For mean-field VI and product prior distributions, we found that the cooling parameter must be chosen proportional to the ratio between the number of observations and neurons to achieve a balance between the data fitting and KL regularizer. With this choice, $\ELBOeta$ converges to a limiting functional that has the same structure as the one given by
\cite{chizat2018global,rotskoff2019global,mei2018mean,tzen2019neural,de2020quantitative}
for empirical risk minimization. We also explained why, in the absence of cooling, the KL term can dominate the data fitting term, typically leading to underfitting of the model, which in practice translates into poor results on all metrics considered. Our work therefore provides a well-grounded theoretical justification for the importance of using a partial tempering in the overparameterized framework, which completes the justifications given by  \cite{wenzel2020good,izmailov2021bayesian,nabarro2021data,noci2021disentangling,laves2021posterior}.  While our theoretical results apply to a neural network with a single hidden layer, we have shown numerically that similar conclusions can be drawn for more general NN architectures.
We emphasize that the introduction of a cooling factor into the Mean-Field VI for BNN is not without implications for the validity of Bayesian inference, and that the conclusions that can be drawn in this framework-in particular, Bayesian uncertainty quantification-must therefore be used with care (even though the accuracy, NLL, and ECE metrics obtained with Mean-Field VI compare favorably to their "classical" ERM learning counterparts).


\bibliographystyle{apalike}
\bibliography{biblio}
\newpage

\newpage
\appendix

\section{Technical result}

\begin{lemma}
  \label{eq:prop_tphi}
  Assume \Cref{ass:unifbound}-\ref{ass:unifboundi} and \Cref{ass:gamma_h}. Then for any $x \in \msx$, the function $\theta \mapsto \tphi(\theta,x)$ is continuous. In addition, there exists $C \geq 0$ such that for any $x \in \msx$ and $\theta \in\ThetaOne$, $\normLigne{\tphi(\theta,x) }\leq C$.
\end{lemma}
\begin{proof}
  Since $\phi(\theta, z, x) = a h(b,x)$ and since by \Cref{ass:gamma_h}, $b \mapsto h(b,x)$ is continuous for any $x \in\msx$, it follows that for any $x\in\msx$, $z \in\rset^d$, $\theta \mapsto \phi(\theta,z,x)$ is continuous on $\ThetaOne$. Using \eqref{eq:bound_phi_z_x} and the condition that $\ThetaOne$ is compact, an application of the Lebesgue dominated convergence theorem implies that for any $x \in \msx$, the function $\theta \mapsto \tphi(\theta,x)$ is continuous. Finally, Eq.~\eqref{eq:bound_phi_z_x} and the condition that $\ThetaOne$ is compact shows that there exists $C \geq 0$ such that for any $x \in \msx$ and $\theta \in\ThetaOne$, $\normLigne{\tphi(\theta,x) }\leq C$.
\end{proof}


\section{Proof of \Cref{prop:posterior_collapse}}\label{sec:proof_posterior_collapse}

We have assumed that $\family_{\Theta}$ is a family of Gaussians with diagonal covariance matrices, and  that $\prior \in \family_{\Theta}$ hence there exists $\btheta_0=(\mu,\sigma)\in \R^{N(\dx+\dy)}\times \R^{N(\dx+\dy)}$ such that $\prior= q_{\btheta_0}$. 
For ease of notations, we work with $\prior$ standard Gaussian:
\begin{equation}\label{eq:prior_formula}
    \prior(\bw) = {\prod_{j=1}^{N}\prod_{l=1}^{\dy}{\cN(a_{i, j}; 0, 1)}} \times \prod_{j=1}^{N}{\prod_{l=1}^{\dx}{\cN(b_{j, l}; 0, 1)}}
\end{equation}
Our results hold for more general parameters for $\prior$ but we fix these ones for convenience of notations.
The posterior $\densityq_{\btheta}\in \family_{\Theta}$ is:
\begin{equation}\label{eq:posterior_formula}
\densityq_{\btheta}(\bw) = {\prod_{j=1}^{N}\prod_{l=1}^{\dy}{\cN(a_{j,l}; \mu_{a_{j,l}}, \sigma_{a_{j,l}}^2)}} \times \prod_{j=1}^{N}{\prod_{l=1}^{\dx}{\cN(b_{j, l}; \mu_{b,j}, \sigma_{b_{j, l}}^2)}}
\end{equation}
We define $a_j = (a_{j, 1}, \dots, a_{j, \dy})\in \R^{\dy}$ and $b_j = (b_{1, j}, \dots, b_{\dx, j})\in \R^{\dx}$ respectively the $j^{th}$ row of the first layer weight matrix and the $j^{th}$ column of the second layer weight matrix. We denote $\mu_{a_j}=(\mu_{a_{j,1}},\dots,\mu_{a_{j,\dy}})\in \R^{\dy}$, $\mu_a=(\mu_{a_1},\dots, \mu_{a_N})\in \R^{N\dy}$. \\
\ \\
Recall that 
\begin{equation}\label{eq:rewrite_elbo}
\ELBO(\btheta) = -\mathcal{L}(\densityq_{\btheta}) - \KL(\densityq_{\btheta}|q_{\btheta_0}),\quad \text{ with }\mathcal{L}(\densityq_{\btheta}) = - \E_{\bw \sim \densityq_{\btheta}}\left[\sum_{i=1}^{p}\log(\likelihood(y_i|x_i,\bw))\right],
\end{equation}
where $\likelihood(y|x,\bw)\propto \exp(-\ell(f_{\bw}(x), y ))$ is defined by \eqref{eq:deflikelihood}.\\ 
\ \\
By the optimality of $\btheta^{\star}$, we have:
\begin{equation}
\ELBO(\btheta^{\star})  \geq  \ELBO(\btheta_0),
\end{equation}
Hence,
\begin{equation}\label{eq:optimality_posterior}
     \KL(\densityq_{\btheta^*}|q_{\btheta_0}) \leq \mathcal{L}(q_{\btheta_0}) - \mathcal{L}(\densityq_{\btheta^*}).
\end{equation}
We now deal separately with the square loss (Case 1) and cross-entropy loss (Case 2).  Throughout, we will often use the notation $\upsigma_j=\upsigma(<b_j, x>)$ for any $j=1,\dots,N$ and a generic point $x\in \X$. Since we have assumed that $\upsigma$ is $L$-Lipschitz, for any $y\in \R$, $|\upsigma(y)|\le |\upsigma(0)| + L|y|.$
Also, to explicit the dependence of $\btheta^*, \btheta_0$ in $N$ we will write their associated distributions $\densityq_{\btheta^*}^N$ and $\densityq_{\btheta_0}^N$ respectively.
\\

\subsection{Case of the square loss}
\begin{proof}
 
 The idea of the proof is to show that the right hand side term of \eqref{eq:optimality_posterior} converges to zero by showing that the two negative log likelihoods converge to the same finite limit, and hence their difference to zero as $N$ goes to infinity. 
 When $l$ is the square loss, for any $q_{\btheta}^N \in \family_{\Theta}$, by \eqref{eq:rewrite_elbo} we have
\begin{equation}
    \mathcal{L}(q_{\btheta}^N) = \sum_{i=1}^p \E_{\bw\sim  q_{\btheta}^N}\left[ \|y_i\|^2 + \|f_{\bw}(x_i)\|^2 -2\ps{y_i,f_{\bw}(x_i)} + \log(Z)\right],
\end{equation}
where $Z$ is the normalization constant of the model defined by \eqref{eq:deflikelihood}. We will show that for both the prior $q_{\btheta_0}^N$ and optimal posterior $q_{\btheta^*}^N$, the first and second moment of the predictive distribution converge to zero as $N$ goes to infinity. 

Under the prior distribution \eqref{eq:prior_formula}, for any $x\in \X$, the first and second moments of the predictive distribution can be written:
\begin{align}
  \E_{\bw\sim \densityq_{\btheta_0}^N}[f_{\bw}(x)] &= \E\left[\frac{1}{N}\sum_{j=1}^{N}{\upsigma_j a_j}\right] = \frac{1}{N} \sum_{j=1}^{N}{\E[\upsigma_j ] \E[a_j]}= 0\\
  \E_{\bw \sim \densityq_{\btheta_0}^N}[\|f_{\bw}(x)\|^2]& = \E_{\bw \sim \densityq_{\btheta_0}^N}\left[\frac{1}{N^2}\left(\sum_{j=1}^{N}{\upsigma_j^2 \|a_j\|^2} + 2 \sum_{j=1}^{N}{\sum_{k<j}{\upsigma_j  \upsigma_k \ps{a_j,a_k}}}\right)\right]\\ 
  &= \frac{1}{N^2} \sum_{j=1}^{N}{\E_{\bw\sim \densityq_{\btheta_0}^N}[\upsigma_j^2]}\le \frac{1}{N^2}\left(\sum_{j=1}^N|\upsigma(0)|^2 + L^2 \|x\|^2 \E_{\bw\sim \densityq_{\btheta_0}^N}[\|b_j\|^2]\right)\\
 & \xrightarrow[N \to \infty]{} 0,
\end{align}
Hence we first obtain: 
\begin{equation}\label{eq:limit_nn_prior}
\lim_{N \rightarrow \infty}
\mathcal{L}(q_{\btheta_0}^N)
=\sum_{i=1}^{p}{\|y_i\|^2} + \log Z.
\end{equation}
We now turn to showing that  $\mathcal{L}(\densityq_{\btheta^*}^N)$ has the same limit. First notice that since $\mathcal{L}$ is a positive function, by \eqref{eq:optimality_posterior} we have: $\KL(\densityq_{\btheta^*}^N|q_{\btheta_0}^N) \leq \mathcal{L}(\densityq_{\btheta_0}^N)$. Since the right-hand term is a converging sequence, it means that $\KL(\densityq_{\btheta^*}^N|q_{\btheta_0}^N)$ is bounded by a constant $C_{\KL}$ independent of $N$.\\
\ \\
By applying  \Cref{lemm:mean_predictive} and  \Cref{lemm:stq_predictive}, we have:
\begin{align}
\E_{\bw \sim \densityq_{\btheta^*}^N}[<y_i,f_{\bw}(x_i)>] &\leq \|y_i\| \|\E_{\bw \sim \densityq_{\btheta^*}^N}[f_{\bw}(x_i)]\| \leq \frac{\phi(\KL(\densityq_{\btheta^*}^N, \densityq_{\btheta_0}^N), \X, \dy)}{\sqrt{N}} \leq \frac{\phi(C_{\KL}, \X, \dy)}{\sqrt{N}}\\
\E_{\bw \sim \densityq_{\btheta^*}^N}[\|f_{\bw}(x)\|^2]&\leq \frac{\psi(\KL(\densityq_{\btheta^*}^N, \densityq_{\btheta_0}^N), \X, \dy)}{\sqrt{N}} \leq \frac{\psi(C_{\KL}, \X, \dy)}{\sqrt{N}}
\end{align}
where the most right hand side inequalities come from the fact that $\KL(\densityq_{\btheta^*}^N|q_{\btheta_0}^N)$ is bounded by a constant $C_{\KL}$ independent of N; and $\phi(C_{\KL}, \X, \dy),\psi(C_{\KL}, \X, \dy)$ are constants that only depend on the data points $(x_i,y_i)_{i=1}^p$, the spaces $\X$, $\Y$ and parameters of the prior distribution (through $C_{\KL}$). 
Hence, the first and second moments of the predictive under the posterior $\densityq_{\btheta^*}^N$ converge to 0. Hence, we obtain:
\begin{equation}\label{eq:limit_nn_posterior}
\lim_{N \rightarrow \infty} \mathcal{L}(q_{\btheta^*}^N)=
\sum_{i=1}^{p}{\|y_i\|^2} + \log Z.
\end{equation} 
From \eqref{eq:optimality_posterior}, \eqref{eq:limit_nn_prior} and  \eqref{eq:limit_nn_posterior} we finally that
\begin{equation}
    \lim_{N \rightarrow \infty} \KL(\densityq_{\btheta^*}^N, \densityq_{\btheta_0}^N) = 0.\qedhere
    \end{equation}
\end{proof}

\subsection{Case of the cross-entropy}
\begin{proof}
Similarly to the square loss case, the idea of the proof is to show that $\mathcal{L}(\densityq_{\btheta_0}^N),\mathcal{L}(\densityq_{\btheta^*}^N) $ have the same limit. We will make use of \Cref{lemm:main_lemma_ce} which specify that limit under a null moment assumption. \\

Under the prior distribution $\densityq_{\btheta_0}^N$,
\begin{equation}
\|\E_{\bw \sim \densityq_{\btheta_0}^N}[\frac{1}{N} \sum_{j=1}^{N}{\upsigma_j a_{j}}]\| = \frac{1}{N} \| \sum_{j=1}^{N}{\E[\upsigma_j]\E[a_j] } \| = 0,
\end{equation}
hence by \Cref{lemm:main_lemma_ce}: $$\lim_{N \rightarrow \infty} \mathcal{L}(\densityq_{\btheta_0}^N) =p( \log (\dy) + \log Z).$$
\ \\
We now turn to the predictive distribution under the posterior $\densityq_{\btheta^*}^N$. Recall  that since $\mathcal{L}$ is a positive function, using the optimality of the posterior we have: $\KL(\densityq_{\btheta^*}^N|\densityq_{\btheta_0}^N) \leq \mathcal{L}(\densityq_{\btheta_0}^N)$. Since the right-hand term is a converging sequence, it means that $\KL(\densityq_{\btheta^*}^N|\densityq_{\btheta_0}^N)$ is bounded by a constant $C_{\KL}$ independent of N.\\
\ \\
 By  \Cref{lemm:mean_predictive}, we can bound the first moment of the predictive distribution as:
\begin{equation}
    \|\E_{\bw \sim \densityq_{\btheta^*}^N}[f_{\bw}(x)]\| \leq \frac{\phi(\KL(\densityq_{\btheta^*}^N, \densityq_{\btheta_0}^N), \X, \dy)}{\sqrt{N}} \leq \frac{\phi(C_{\KL}, \X, \dy)}{\sqrt{N}},
    \end{equation}
    where the last inequality comes from the fact that the KL term is bounded by a constant $C_{\KL}$ independent of $N$ for the optimal variational parameter $\btheta^{*,N}$. 
Moreover, by using similar argument than in the proof of Lemma \ref{lemm:mean_predictive}, we can show that each coordinate $\mu,\sigma$ of $\btheta^{*,N}$ is bounded as:
\begin{itemize}
    \item $\mu \leq \sqrt{2 \KL(\densityq_{\btheta^*}^N, \densityq_{\btheta_0}^N)} \leq \sqrt{2 C_{\KL}}$
    \item $\sigma \leq 2 \KL(\densityq_{\btheta^*}^N, \densityq_{\btheta_0}^N) + 1 \leq 2 C_{\KL} + 1$
\end{itemize}
It means that each neuron weight has bounded mean and variance. 
We can thus apply \Cref{lemm:main_lemma_ce}, which yields: $$\lim_{N \rightarrow \infty}  \mathcal{L}(q_{\btheta}^N) = p(\log(\dy) - \log Z).$$

As $0 \leq \KL(\densityq_{\btheta^*}^N|\densityq_{\btheta_0}^N) \leq \mathcal{L}(\densityq_{\btheta_0}^N) - \mathcal{L}(\densityq_{\btheta^*})$ we obtain: 
\begin{equation}
    \lim_{N \rightarrow \infty} \KL(\densityq_{\btheta^*}^N, \densityq_{\btheta_0}^N) = 0. \qedhere
\end{equation}
\end{proof}

\begin{lemma}\label{lemm:mean_predictive} 
Assume the conditions of \Cref{prop:posterior_collapse} hold.
Then there exists a function $\phi$, increasing in its first variable, such that
$$\|\E_{\bw \sim \densityq_{\btheta}^N}[f_{\bw}(x)]\| \leq \frac{\phi(\KL(\densityq_{\btheta}^N,\densityq_{\btheta_0}^N), \X, \dy)}{\sqrt{N}}.$$

\end{lemma}
\begin{proof}
By Cauchy-Schwartz inequality, the first moment of the predictive distribution under the variational posterior can be upper bounded as:
\begin{equation}
\|\E_{\bw \sim \densityq_{\btheta}^N}[f_{\bw}(x)]\| =\frac{1}{N} \|\E_{\bw \sim \densityq_{\btheta}^N}[ \sum_{j=1}^{N}{\upsigma(\ps{b_j,x}) a_j}]\|
\leq \frac{1}{N} \sum_{j=1}^{N}{|\E[\upsigma(\ps{b_j,x})]| \|\mu_{a_j}\|}. 
\end{equation}
Since $\upsigma$ is Lipschitz, $|\upsigma(x)| \leq C_0 + L |x|$ where $C_0=|\upsigma(0)|$. Hence,
\begin{equation}
|\E[\upsigma(\ps{b_j,x})] | \leq | C_0 + L \E[|\ps{b_j, x}|] | \leq C_0 + L \sum_{l=1}^{\dx} \E[|b_{j, l}|] |x_l|]
\end{equation}
Let's start by finding an upper bound for $\E[|b_{j, l}|]$. If $b_{j, l} \sim \cN(\mu_{b_{j, l}}, \sigma_{b_{j, l}}^2)$, then $|b_{j, l}|$ has an absolute Gaussian distribution and denoting $\Phi$ the CDF of a standard Gaussian, we have
\begin{equation}
   \E[|b_{j, l}|] = \sigma_{b_{j, l}} \sqrt{\frac{2}{\pi}} \exp\left(\frac{-\mu_{b_{j, l}}^2}{2 \sigma_{b_{j, l}}^2}\right) + \mu_{b_{j, l}} \left[ 1 - 2 \Phi\left(- \frac{\mu_{b_{j, l}}}{\sigma_{b_{j, l}}}\right)\right]  \leq \sigma_{b_{j, l}} \sqrt{\frac{2}{\pi}} + |\mu_{b_{j, l}}|.
\end{equation}
Recall that the $\KL$ between the posterior $\densityq_{\btheta}^N$ and prior $\densityq_{\btheta_0}^N$ can be written:
\begin{multline}\label{eq:def_kl_gaussian} \KL(\densityq_{\btheta}^N| \densityq_{\btheta_0}^N)=\frac{1}{2}\sum_{j=1}^{N}\left[\sum_{l=1}^{\dx}( \mu_{b_{j, l}}^2
+ \sigma_{b_j,l}^2 - \log(\sigma_{b_j,l}^2) - 1)\right. \\\left.+ \sum_{l=1}^{\dy}( \mu_{a_{j, l}}^2 + \sigma_{a_j,l}^2 - \log(\sigma_{a_j,l}^2) - 1) \right]
\end{multline} 
Hence, for any $j=1,\dots,N$ and $l=1,\dots, \dx$:
\begin{align}\label{eq:use_def_kl_gaussian_mean}
|\mu_{b_{j,l}}|&\le \|\mu\|_2\le \sqrt{2\KL(\densityq_{\btheta}^N|\densityq_{\btheta_0}^N)}, \\
\label{eq:use_def_kl_gaussian_var}
\sigma_{b_{j,l}}\le &\|\sigma\|_2\le |\sigma_{b_{j,l}}+1-1| \le |\sigma_{b_{j,l}}^2-\log(\sigma_{b_{j,l}}^2)-1| + 1\le 2\KL(\densityq_{\btheta}^N|\densityq_{\btheta_0}^N)+1,
\end{align}
and 
\begin{equation}\label{eq:D_fct}
\E[|b_{j, l}|] \leq \sqrt{\frac{2}{\pi}} (2 \KL(\densityq_{\btheta}^N|\densityq_{\btheta_0}^N) + 1) + \sqrt{2 \KL(\densityq_{\btheta}^N|\densityq_{\btheta_0}^N)} := D(\KL(\densityq_{\btheta}^N|\densityq_{\btheta_0}^N))
\end{equation}
Where D is increasing. Hence, since $\X$ is compact, there exists $C_{\X}$ such that $\|x\|_{1}\le C_{\X}$ and:
\begin{equation}\label{eq:H_fct}
    |\E[\upsigma(\ps{b_j, x})]| \leq
C_0 + L C_{\X} D(\KL(\densityq_{\btheta}^N|\densityq_{\btheta_0}^N)):=E(\KL(\densityq_{\btheta}^N|\densityq_{\btheta_0}^N), \X),
\end{equation}
Where $E$ is increasing in its first variable. Finally, since \begin{equation}
     N^{-1}\sum_{j=1}^N \|\mu_{a_j}\|_2\le N^{-1}\|\mu_a\|_{1}\le N^{-1}\sqrt{N\dy} \|\mu_a\|_2\le N^{-\frac{1}{2}}\sqrt{\dy}\sqrt{2\KL(q_{\btheta}^N, \densityq_{\btheta_0}^N)},
 \end{equation} the first moment of the predictive distribution can be upper bounded as:
\begin{align}
    \|\E_{\bw \sim \densityq_{\btheta}}[f_{\bw}(x)\| 
    & \leq \frac{E(\KL(\densityq_{\btheta}^N|\densityq_{\btheta_0}^N), X)\sqrt{\dy} \sqrt{2\KL(q_{\btheta}^N, \densityq_{\btheta_0}^N)}}{\sqrt{N}}:= \frac{\phi({\KL(q_{\btheta}^N, \densityq_{\btheta_0}^N), \X, \dy)}}{\sqrt{N}},
\end{align}
where $\phi$ is increasing in its first variable.
\end{proof}

\begin{lemma}\label{lemm:stq_predictive}
Assume the conditions of \Cref{prop:posterior_collapse} hold.
Then there exists a function $\psi$ depending only on $\KL(\densityq_{\btheta}^N,\densityq_{\btheta_0}^N)$, $\X$, and $\dy$ such that $G$, increasing in its first variable, such that:
$$\E_{\bw \sim \densityq_{\btheta}^N}[\|f_{\bw}(x)\|^2] \leq \frac{\psi(\KL(\densityq_{\btheta}^N,\densityq_{\btheta_0}^N), \X, \dy)}{N}.$$
\end{lemma}
\begin{proof}
For a posterior of the form \eqref{eq:posterior_formula}, we can write the second moment of the predictive distribution as:
\begin{align}\label{eq:decompose_second_moment}
    \E_{\bw \sim \densityq_{\btheta}^N}[\|f_{w}(x)\|^2] 
    &= \frac{1}{N^2} \sum_{j=1}^{N}{\E[\upsigma_j^2]\E[\|a_j\|^2]} + \frac{2}{N^2} \sum_{j=1}^{N}\sum_{k<j}^N\E[\upsigma_j] \E[\upsigma_k] \E[\ps{a_j, a_k}].
\end{align}
We start with the second term on the right hand side of \eqref{eq:decompose_second_moment}.
Using $\E[\ps{a_j, a_k}]=\ps{\mu_{a_j},\mu_{a_k}}\le1/2 (\|\mu_{a_j}\|^2+\|\mu_{a_k}\|^2)$, along with \eqref{eq:use_def_kl_gaussian_mean} and \eqref{eq:H_fct}, we have
\begin{align}
    \frac{1}{N^2} \sum_{j=1}^{N}\sum_{k=1}^{N}{\E[\upsigma_j] \E[\upsigma_k] \ps{\mu_{a_j}, \mu_{a_k}}} &\leq \frac{E^2(\KL(\densityq_{\btheta}^N, \densityq_{\btheta_0}^N)) 2\KL(\densityq_{\btheta}^N, \densityq_{\btheta_0}^N) }{N^2}.  
\end{align}
We now turn to the first term on the right hand side of \eqref{eq:decompose_second_moment}. We first have for any $j=1,\dots,N$, using \eqref{eq:def_kl_gaussian} that:
\begin{equation}
    \E[\|a_{j}\|^2] =\sum_{l=1}^{\dy} \E[a_{j,l}^2]= \sum_{l=1}^{\dy} (\sigma_{a_{j,l}}^2+\mu^2_{a_{j,l}})  \leq 2 \KL(\densityq_{\btheta^*}^N, \densityq_{\btheta_0}^N) + \dy (2 \KL(\densityq_{\btheta^*}^N, \densityq_{\btheta_0}^N) + 1)^2
    := F(\KL(\densityq_{\btheta^*}^N, \densityq_{\btheta_0}^N)).
\end{equation}
Then, using that $\upsigma$ is $L$-Lipschitz, Cauchy-Schwartz inequality and that since $\X$ is compact there exists $c_{\X}$ such that $\|x\|\le c_{\X}$, we have:
\begin{align}
    \E[\upsigma_j^2]  \le \E[(C_0 + L |<b_j, x>|)^2] 
    = C_0^2 + 2 C_0 c_{\X}L \E[\|b_j\| ] + L^2 c^2_{\X} \E[\|b_j\|^2],
\end{align}
where, using \eqref{eq:def_kl_gaussian} and \eqref{eq:D_fct}, 
\begin{align}
    \E[\|b_j\| ]&\le \sum_{l=1}^{\dx}\E[|b_{j,l}|] \le \dx D(\KL(\densityq_{\btheta}^N|\densityq_{\btheta_0}^N)),\\
    \E[\|b_j\|^2] &= \sum_{l=1}^{\dx}\E[b_{j,l}^2] =  \sum_{l=1}^{\dx} (\sigma_{b_{j,l}}^2+\mu^2_{b_{j,l}} )\le \dx (2\KL(\densityq_{\btheta^*}^N, \densityq_{\btheta_0}^N) + 1)^2 + 2 \KL(\densityq_{\btheta}^N|\densityq_{\btheta_0}^N).
\end{align}
Hence,
\begin{align}
\E[\upsigma_j^2] &\leq C_0^2+2C_0 c_{\X} L \dx D(\KL(\densityq_{\btheta}^N|\densityq_{\btheta_0}^N))+ L^2 c_{\X}^2 2\KL(\densityq_{\btheta}^N|\densityq_{\btheta_0}^N) :=G(\KL(\densityq_{\btheta^*}^N, \densityq_{\btheta_0}^N)),
\end{align}
with $R$ increasing. Hence, the first term on the right hand side of \eqref{eq:decompose_second_moment} can be bounded as:
\begin{align}
    \frac{1}{N^2} \sum_{j=1}^{N}{\E[\upsigma_j^2]\E[\|a_j\|^2]} 
    \leq \frac{G(\KL(\densityq_{\btheta^*}^N, \densityq_{\btheta_0}^N)) F(\KL(\densityq_{\btheta^*}^N, \densityq_{\btheta_0}^N))}{N}.
\end{align}
Finally, we obtain the desired result with:
\begin{equation}
\psi(\KL(\densityq_{\btheta^*}^N, \densityq_{\btheta_0}^N)) := G(\KL(\densityq_{\btheta^*}^N, \densityq_{\btheta_0}^N)) F(\KL(\densityq_{\btheta^*}^N, \densityq_{\btheta_0}^N)) + E^2(\KL(\densityq_{\btheta^*}^N, \densityq_{\btheta_0}^N)) \sqrt{2\KL(\densityq_{\btheta^*}^N, \densityq_{\btheta_0}^N)}.
\end{equation}
\end{proof}

\begin{lemma}\label{lemm:main_lemma_ce}
Let $l$ be the cross-entropy loss, and $q_{\btheta}^N \in \family_{\Theta}$ where $\family_{\Theta}$ is a family of  Gaussians with diagonal covariance matrices, i.e. for any $\btheta\in \Theta$, $\btheta=(\mu,\sigma) \in\R^{N\dx}\times R^{N\dy}$. Assume that each coordinate of $\btheta$ is bounded by a constant (independent of $N$)  and that $\lim_{N \rightarrow \infty} \|\E_{\bw \sim q_{\btheta}}[f_{\bw}(x)]\| = 0$ for any $x\in \X$. Then, $$\lim_{N \rightarrow \infty} 
\mathcal{L}(\densityq_{\btheta}^N)
= p(\log (\dy) + \log(Z)).$$
\end{lemma}
\begin{proof}
 For any $i=1,\dots,p$, denote
\begin{equation}
\begin{array}[t]{lrcl}
 l_{y_i} : & \R^{\dy} & \longrightarrow & \R \\
    & (z_1, \dots, z_{\dy})  &  \longmapsto &  - \log\left(\frac{e^{z_{y_i}}}{\sum_{j=1}^{\dy}{e^{z_j}}}\right) \end{array},
\end{equation}
so that $\forall z = (z_1, \dots z_{\dy}) \in \R^{\dy}$,
\begin{equation}\label{eq:decompose_cross_entropy}
    |l_{y_i}(z)| = \left|-\log(\exp(z_{y_i}) + \log\left(\sum_{k=1}^{\dy}{\exp(z_k)}\right)\right|.
\end{equation}

By the definition of $\mathcal{L}$  and plugging  $-\log(\dy) - \log(Z)$ in \eqref{eq:decompose_cross_entropy}, we have:
\begin{align}
    |
    \mathcal{L}(\densityq_{\btheta}^N)
    - p(\log(\dy)+\log(Z))|
    &\leq \sum_{i=1}^p|\E_{\bw \sim \densityq_{\btheta}^N}[f_{\bw}(x, y_i)]| + |\E_{\bw \sim \densityq_{\btheta}^N}[\log \frac{1}{\dy}\sum_{k=1}^{\dy}{e^{f_{\bw}(x, k)}}]|,
\end{align}
where $f_{\bw}(x,k)$ denotes the $k$-th coordinate of $f_{\bw}(x)\in \R^{\dy}$ for $l=1,\dots,\dy$. 
The first term on the right hand side of the previous inequality converges to 0 as $N$ goes to infinity by assumption. Hence, we can focus on the second term. For any $k=1,\dots,\dy$, since $\upsigma$ is $L$-Lipschitz,
\begin{equation}
    f_{\bw}(x, k) = \frac{1}{N} \sum_{j=1}^{N}{\upsigma(<b_j, x>) a_{j, k}} \leq  \frac{1}{N} \sum\limits_{j=1}^{N} C_0 a_{j, k} + 
 \frac{L}{N} \sum\limits_{j=1}^{N}{\sum\limits_{l=1}^{\dx}{|b_{j, l}||x_l| a_{j, k}}}.
\end{equation}
Using the previous inequality along with Jensen's inequality, we have
\begin{equation}
\left|\E_{\bw \sim \densityq_{\btheta}^N}\left[\log \left(\frac{1}{\dy}\sum_{k=1}^{\dy}{e^{\frac{L}{N} \sum\limits_{j=1}^{N}{\sum\limits_{l=1}^{\dx}{|b_{j, l}||x_l| a_{j, k}}}}}\right)\right]\right| \leq \left|\log \frac{1}{\dy} \sum_{k=1}^{\dy}{\prod_{j=1}^{N}{\prod_{l=1}^{\dx}{\E_{\bw \sim \densityq_{\btheta}^N}\left[e^{\frac{L|b_{j, l}| |x_l| a_{j, k}}{N}}\right]}}}\right|.
\end{equation}
Since the posterior is of the form \eqref{eq:posterior_formula} we have for any index $(j,k,l)$:
\begin{align}
    \E_{\bw \sim \densityq_{\btheta}^N}\left[e^{L\frac{|b_{j, l}| |x_l| a_{j, k}}{N}}\right] &\leq \E_{\bw \sim \densityq_{\btheta}^N}\left[e^{L\frac{|b_{j, l}| |x_l| |a_{j, k}|}{N}}\right] \\
   & = \E_{u \sim \cN(0, 1); v \sim \cN(0, 1)}\left[e^{\frac{L|\sigma_{b_{j,l}} u + \mu_{b_{j,l}}| |x| |\sigma_{a_{j,k}} v + \mu_{a_{j,k}}|}{N}}\right]\\
   &\le \E_{u \sim \cN(0, 1); v\sim \cN(0, 1)}\left[e^{\frac{|C b + C| |x| | C a + C|}{N}}\right],
    \end{align}
for some constant $C>0$ since by assumption each coordinate of the variational parameter is bounded.    
By the dominated convergence theorem, when $N$ goes to infinity we have:
\begin{equation}
  \E_{u \sim \cN(0, 1); v \sim \cN(0, 1)}\left[e^{\frac{L|\sigma_{b_{j,l}} u + \mu_{b_{j,l}}| |x| (\sigma_{a_{j,k}} v + \mu_{a_{j,k}})}{N}}\right]   = 1 + o\left(\frac{1}{N}\right).
    \end{equation}
Hence, 
\begin{equation}
    \left|\E_{\bw \sim \densityq_{\btheta}^N}\left[\log \frac{1}{\dy}\sum_{k=1}^{\dy}{e^{\frac{L}{N} \sum\limits_{j=1}^{N}{\sum\limits_{l=1}^{\dx}{|b_{j, l}||x_l| a_{j, k}}}}}\right]\right| \leq N \dx \log\left(1 + o\left(\frac{1}{N}\right)\right)
\end{equation}
Similarly, we can prove that:
\begin{equation}
    \lim_{N \rightarrow \infty} \left|\E_{\bw \sim \densityq_{\btheta}^N}\left[\log \frac{1}{\dy}\sum_{k=1}^{\dy}{e^{\frac{\upsigma(0)}{N} \sum\limits_{j=1}^{N}{ a_{j, k}}}}\right]\right| = 0
\end{equation}
Finally, we have: 
\begin{equation}
    \lim_{N \rightarrow \infty} |\mathcal{L}(\densityq_{\btheta}^N) - p (\log(\dy)+\log(Z))| \leq \lim_{N \rightarrow \infty} \left|N \dx \log\left(1 + o\left(\frac{1}{N}\right)\right)\right| = 0. 
\end{equation}
\end{proof}

\section{Additional Experiments}\label{sec:additional_experiments}

\subsection{Balanced ELBO with cooling} 
We first support with a very simple experiment the theoretical results of \Cref{sec:bnn_infinite} and the relevance of the form of the parameter $\etaN = \tau p/N$ we find. This experiment does not require training, since the goal here is to illustrate how introducing this parameter allows to balance the contributions of the two terms in the decomposition of  $\ELBOeta$ in \eqref{eq:elboeta}.
We choose the architecture of a one hidden layer neural network with RelU activation functions, to which we will refer to as \textit{Linear BNN}. We consider a regression task on the Boston dataset and a classification task on MNIST. We choose a zero-mean Gaussian prior with variance $1/5$ for each neuron. Also, we initialize the variational parameters $\theta=(\mu,\sigma)$ where $\mu$ is close to zero and $\sigma=10^{-3}$.
\Cref{fig:balanced_mnist} and \ref{fig:balanced_boston} illustrate the ratio between the likelihood and KL terms in $\ELBOeta$ when the number of weights grows, for  $\etaN=1$ (no cooling), $\etaN=\tau p/N$ and different values of the hyperparameter $\tau$, on the MNIST and BOSTON datasets respectively. They confirm that when the number of data points $p$ is fixed and $\ELBOeta$ is not rescaled, one of the two terms become dominant contrary to the case where we set $\etaN=\tau p/N$. 

\begin{figure}[H]
   \centering
\begin{tabularx}{\columnwidth}{cccc}
	\includegraphics[width=.22\columnwidth]{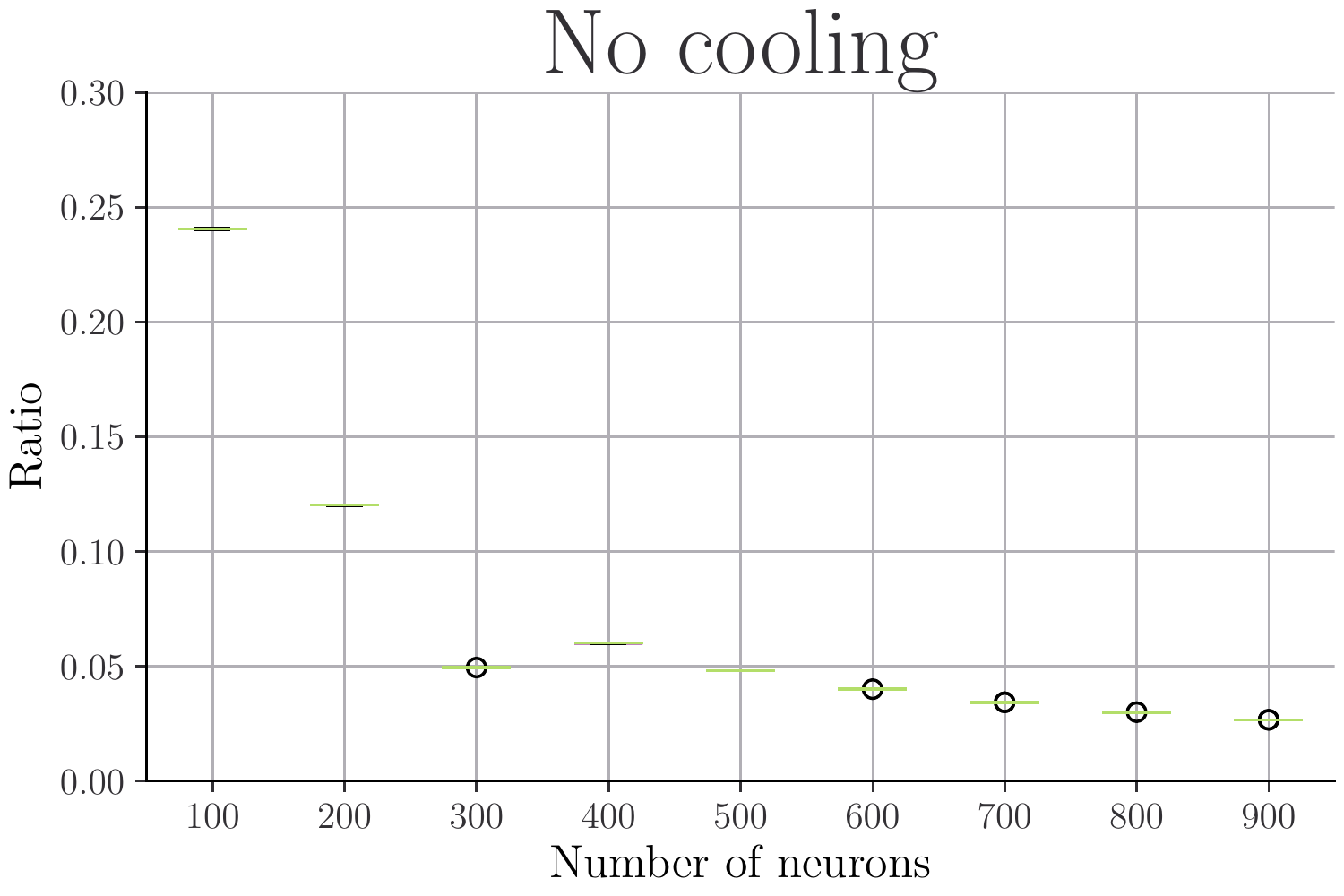}&
	\includegraphics[width=.22\columnwidth]{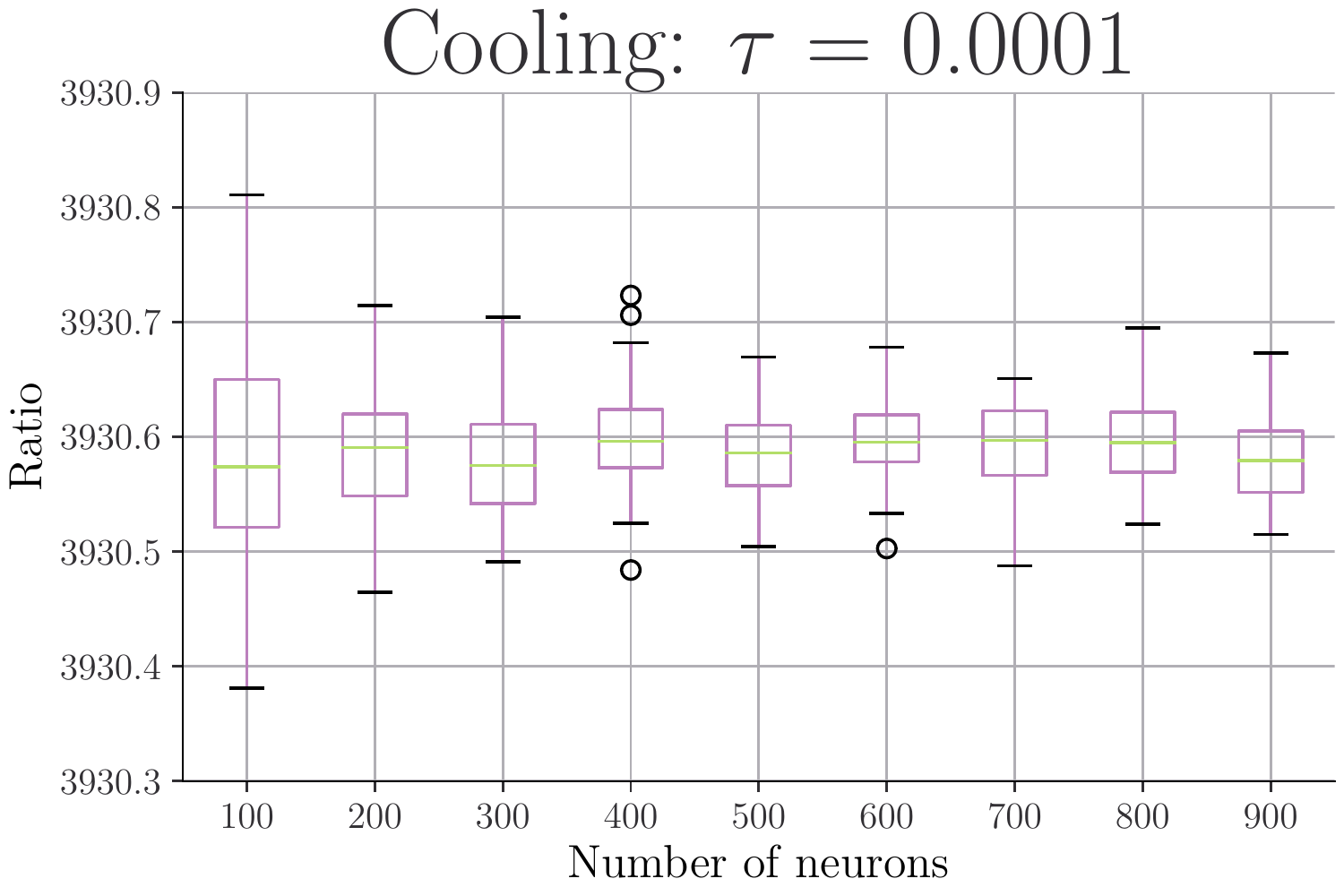}&
  \includegraphics[width=0.22\columnwidth]{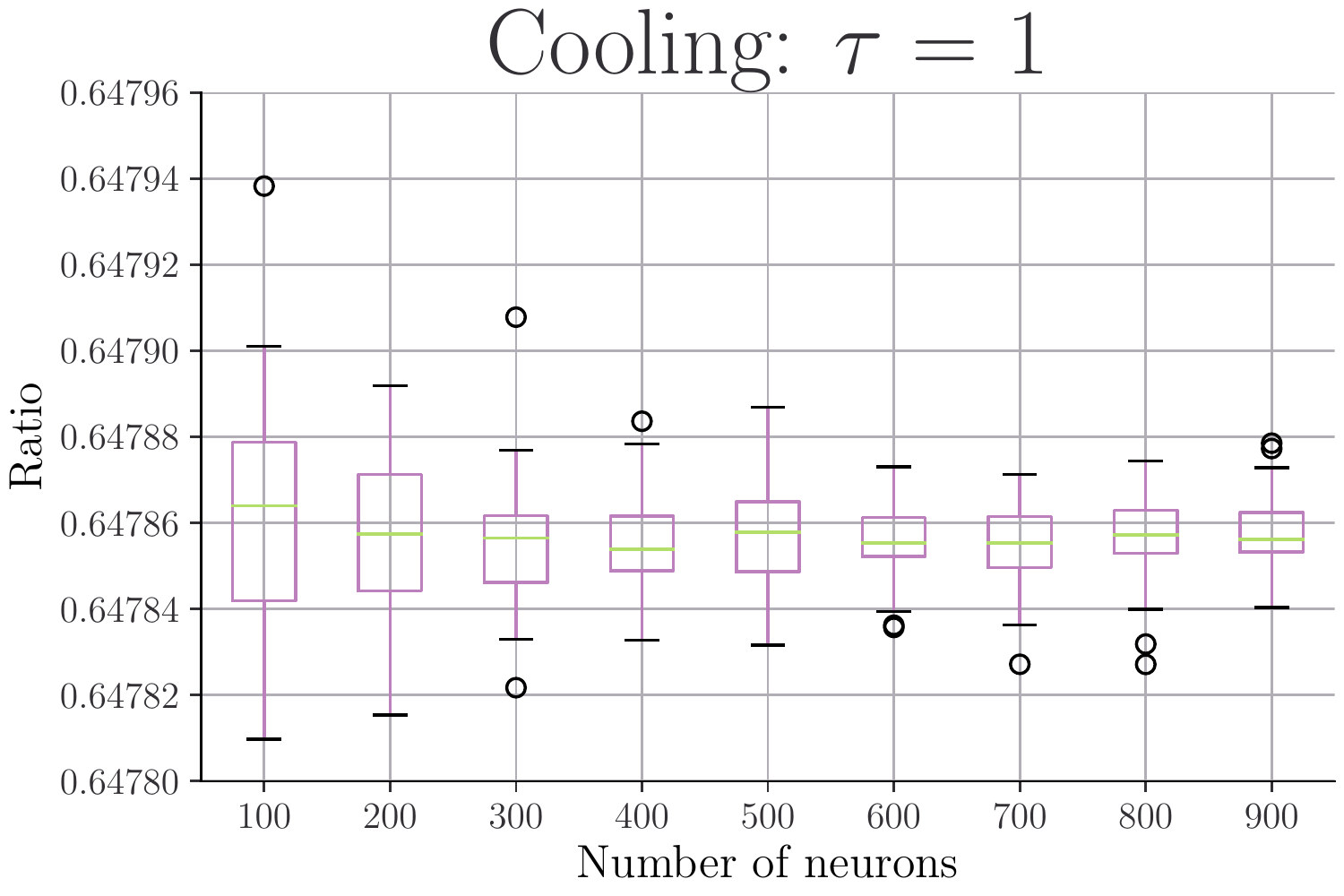}
&\includegraphics[width=.22\columnwidth]{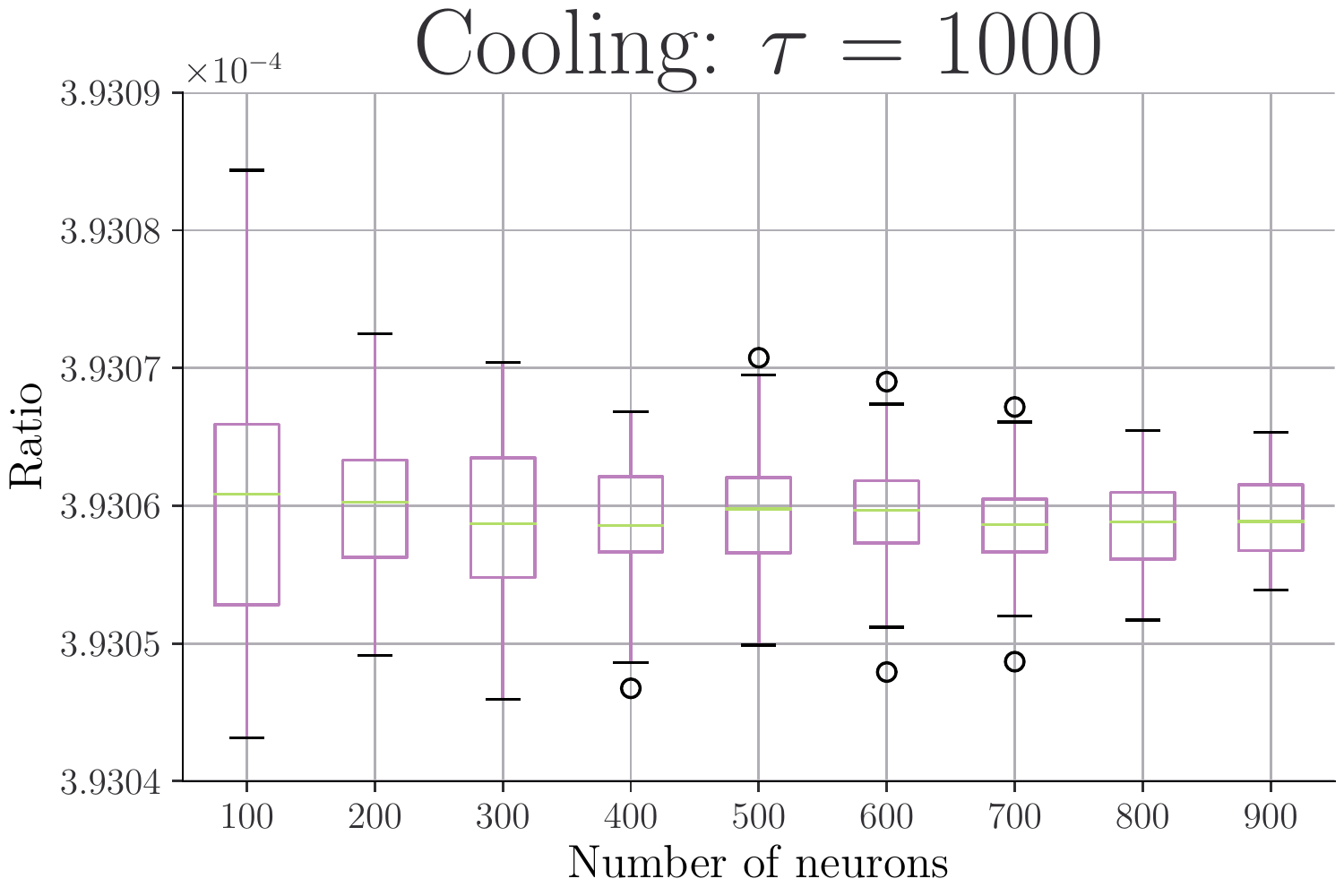}
\end{tabularx}
    \caption{Ratio of the two $\ELBO$ terms,  for a Linear BNN (non trained) on MNIST.}
    \label{fig:balanced_mnist}
\end{figure}

\begin{figure}[H]
   \centering
\begin{tabularx}{\columnwidth}{cccc}
	\includegraphics[width=.22\columnwidth]{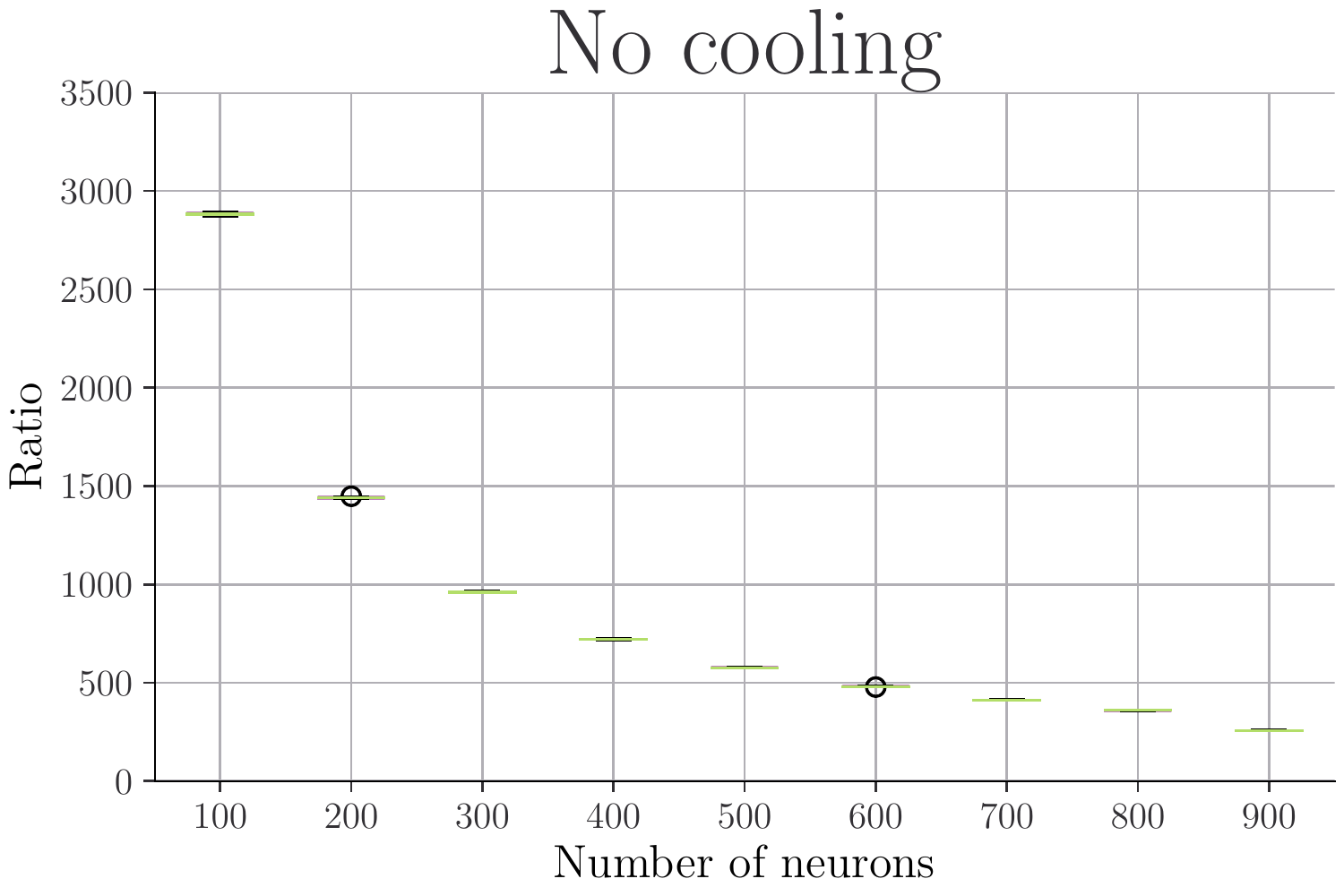}&
	\includegraphics[width=.22\columnwidth]{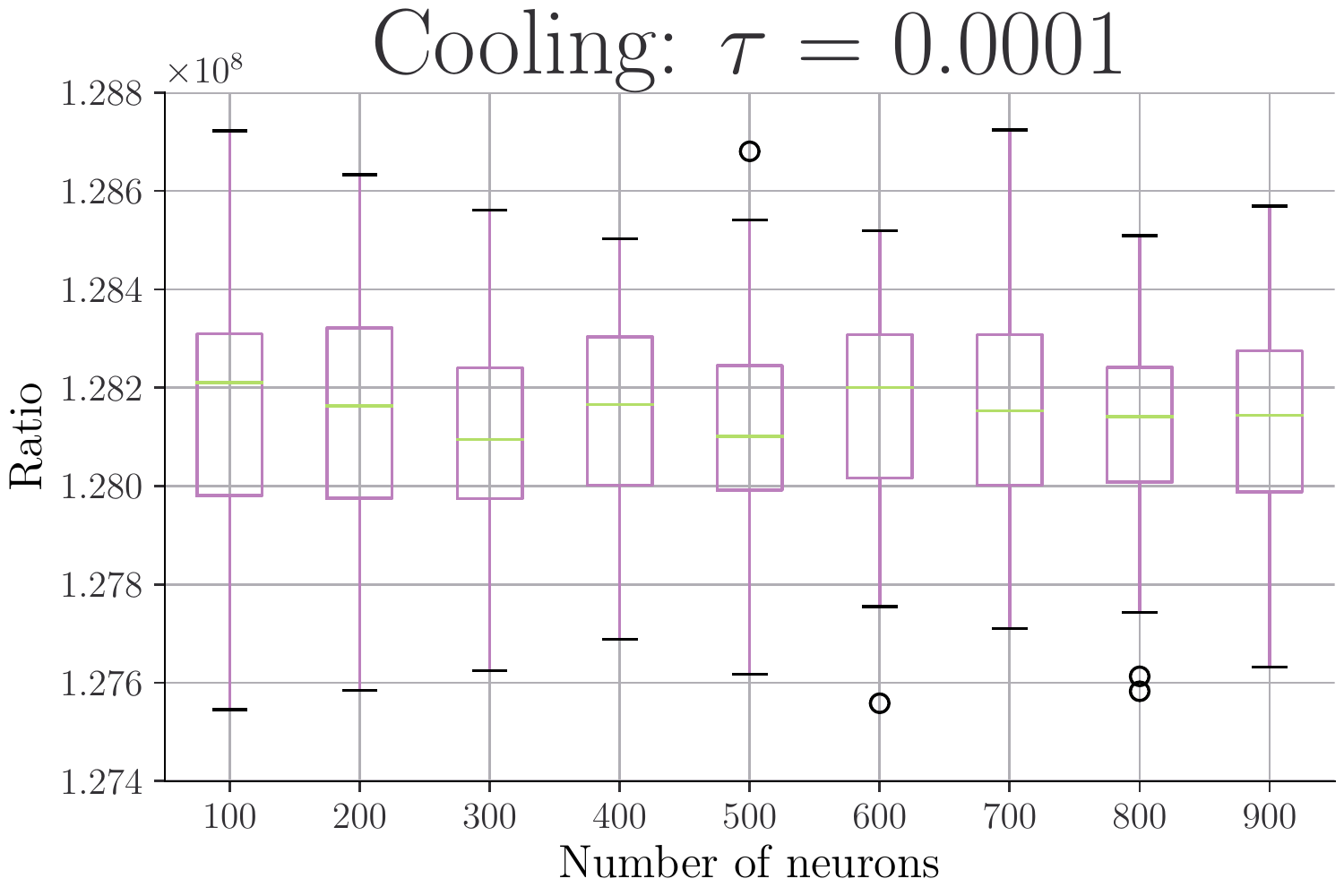}&
  \includegraphics[width=0.22\columnwidth]{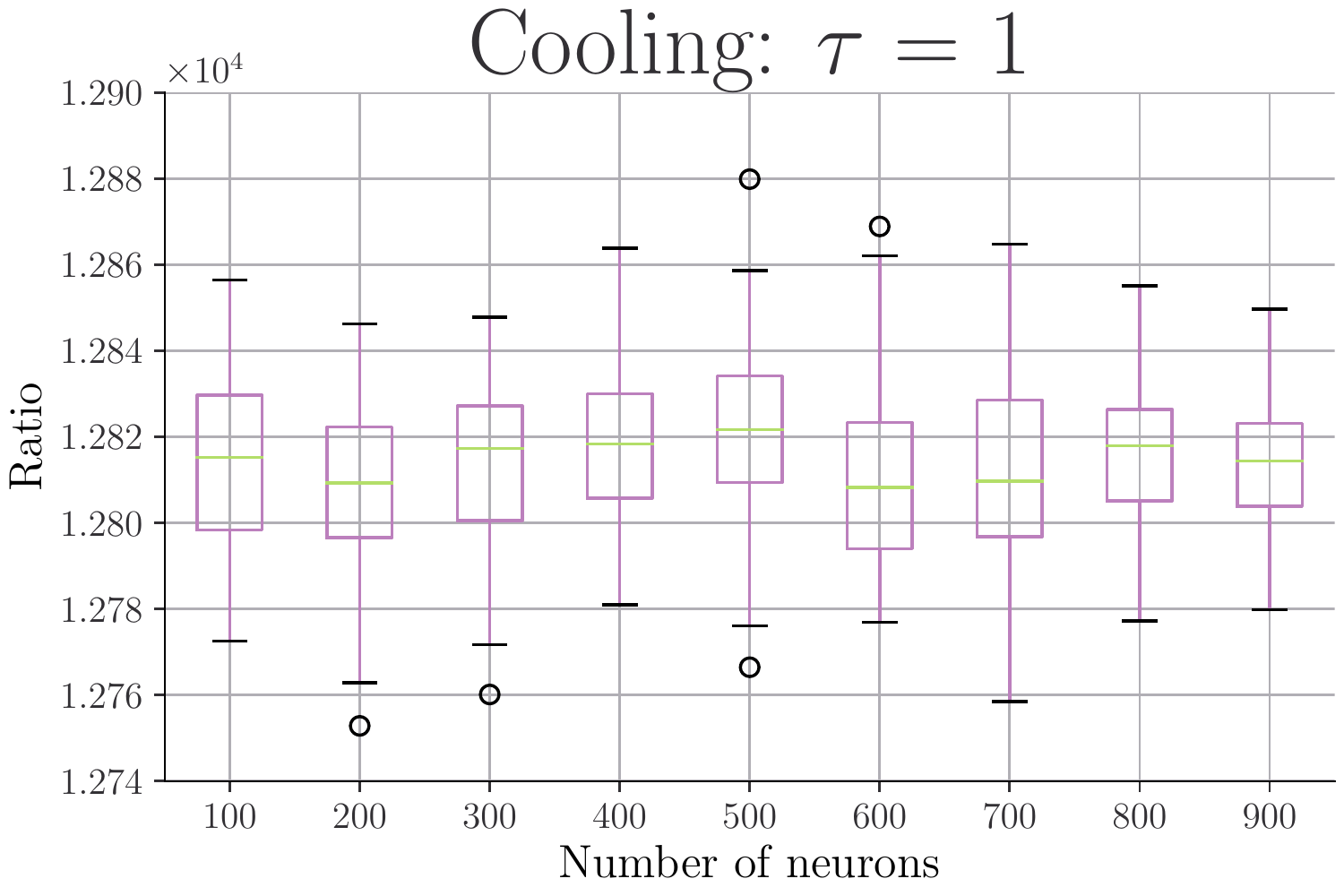} &\includegraphics[width=.22\columnwidth]{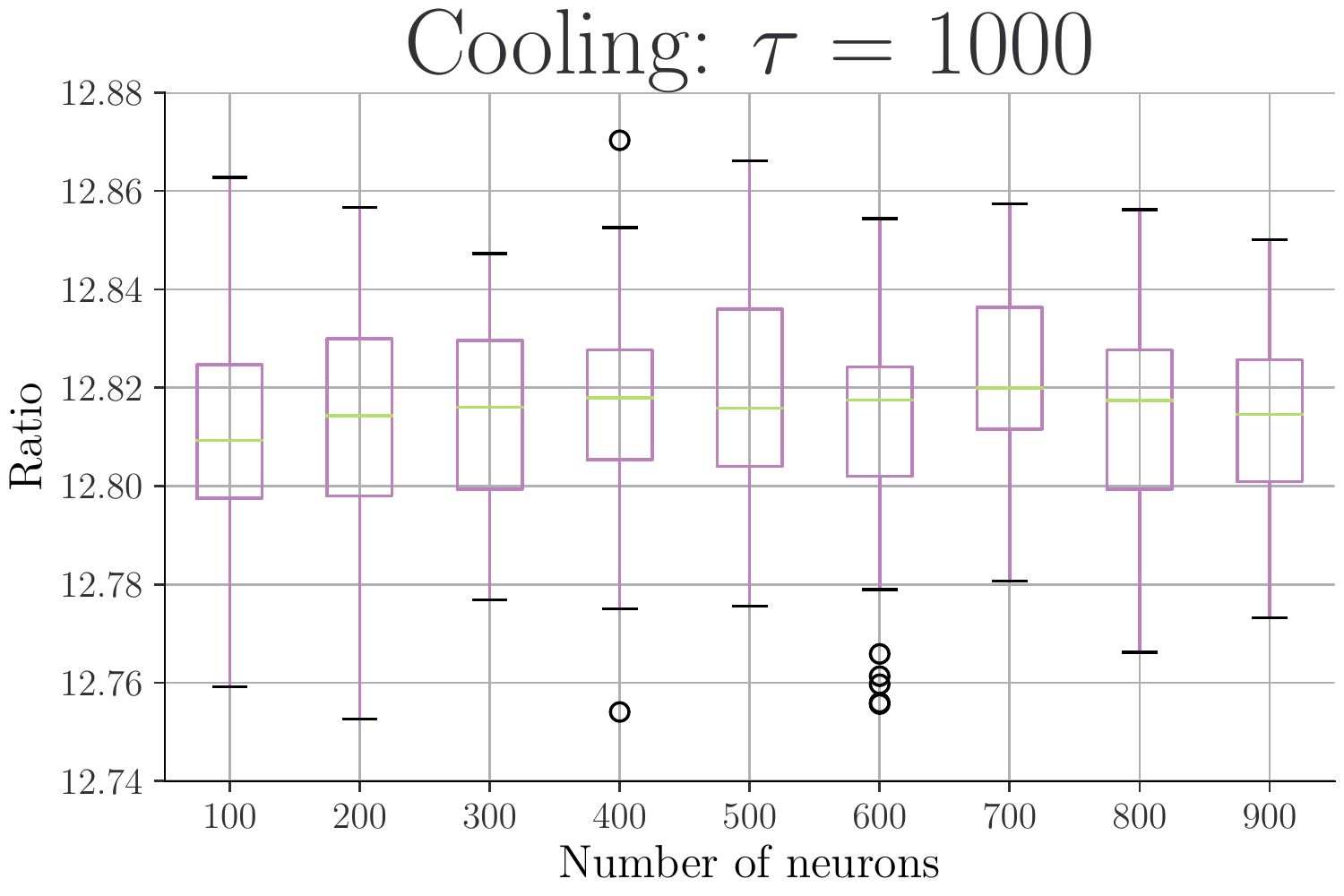}
\end{tabularx}
    \caption{Ratio of the two $\ELBO$ terms,  for a Linear BNN (non trained) on BOSTON.}
    \label{fig:balanced_boston}
\end{figure}

\subsection{ECE definition}\label{sec:def_ece}

For any input $x$, define $\text{conf}(x)=\max_{c\in \{1,\dots,n_l\}} \Psi_c(f_{\bw} (x))$, i.e., the maximal predicted probability of the network. This quantity can be viewed  as a
prediction confidence for the input $x$. ECE
discretizes the interval $[0,1]$ into a given number
of bins $B$ and groups predictions based on the confidence score: $S_b=\{i\in\{1,\dots,p\}, \text{conf}(x_i)\in [b/B, (1+b)/B[  \}$.
The calibration error is the difference between the fraction of predictions in the bin that
are correct (accuracy) and the mean of the probabilities
in the bin (confidence).
\begin{equation}\label{eq:ece}
\text{ECE} = \sum_{b=1}^B \frac{|S_b|}{p} |\text{acc}(S_b)-\text{conf}(S_b)| \eqsp,
\end{equation}
where $p$ is the total number of data points, and $|S_b|$, $\text{acc}(S_b)$
and $\text{conf}(S_b)$ are the number of predictions, the accuracy and confidence of bin $S_b$ respectively.

\subsection{Cooling effect on the distribution of the variational parameters} 
\Cref{fig:weights} illustrates the distribution of the variational parameters after training a linear BNN (i.e., single hidden layer with RelU) on MNIST. For a large $\tau$, the distribution of the variational parameters is close to the prior (a centered Gaussian with standard deviation $0.2$). For a small $\tau$, we can see that the network has learnt values of $\sigma$ that are very different from the prior (e.g., close to zero). Intermediate values of $\tau$ interpolate between the two previous regimes

\begin{figure}[H]
\centering
\subfigure{\includegraphics[width=0.22\columnwidth]{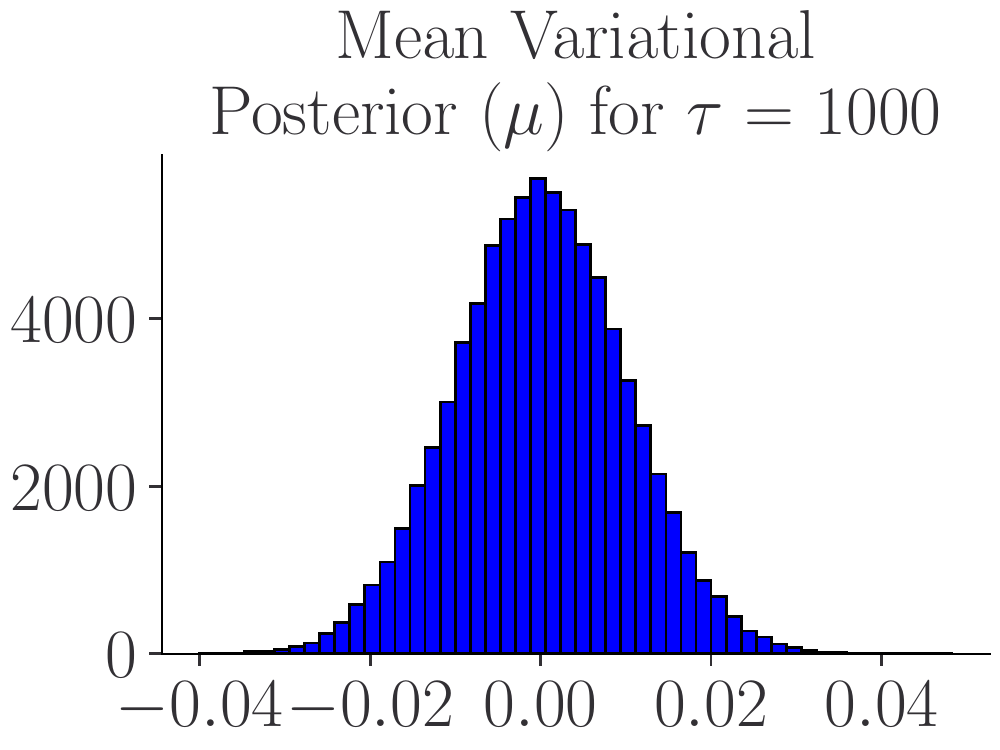}}
\subfigure{\includegraphics[width=0.22\columnwidth]{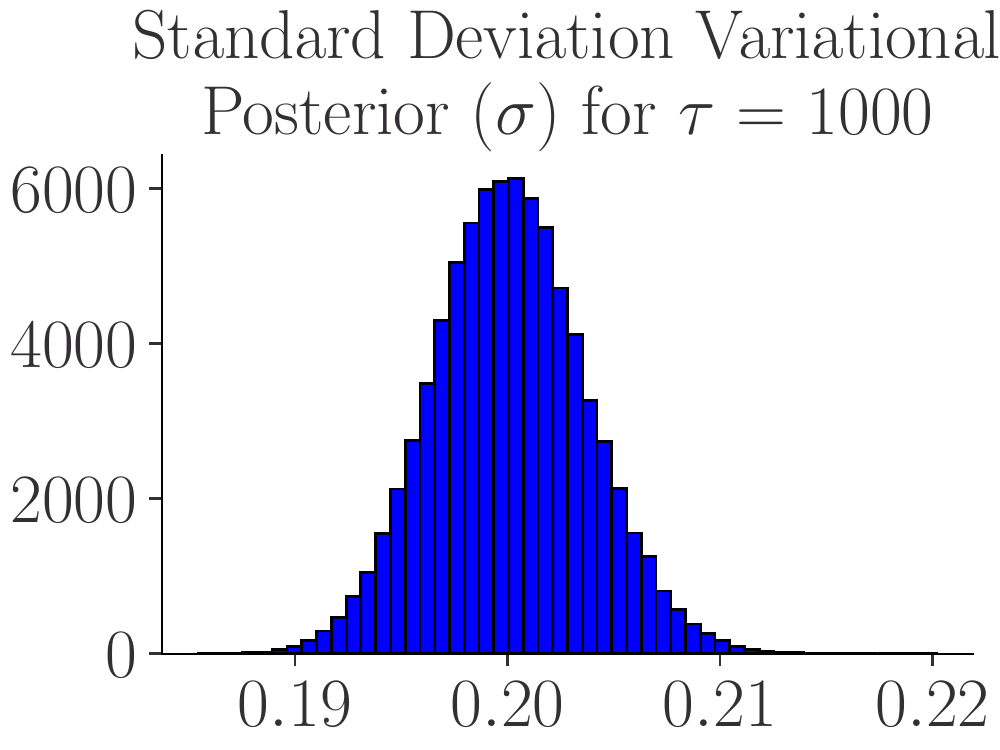}}
\subfigure{\includegraphics[width=0.22\columnwidth]{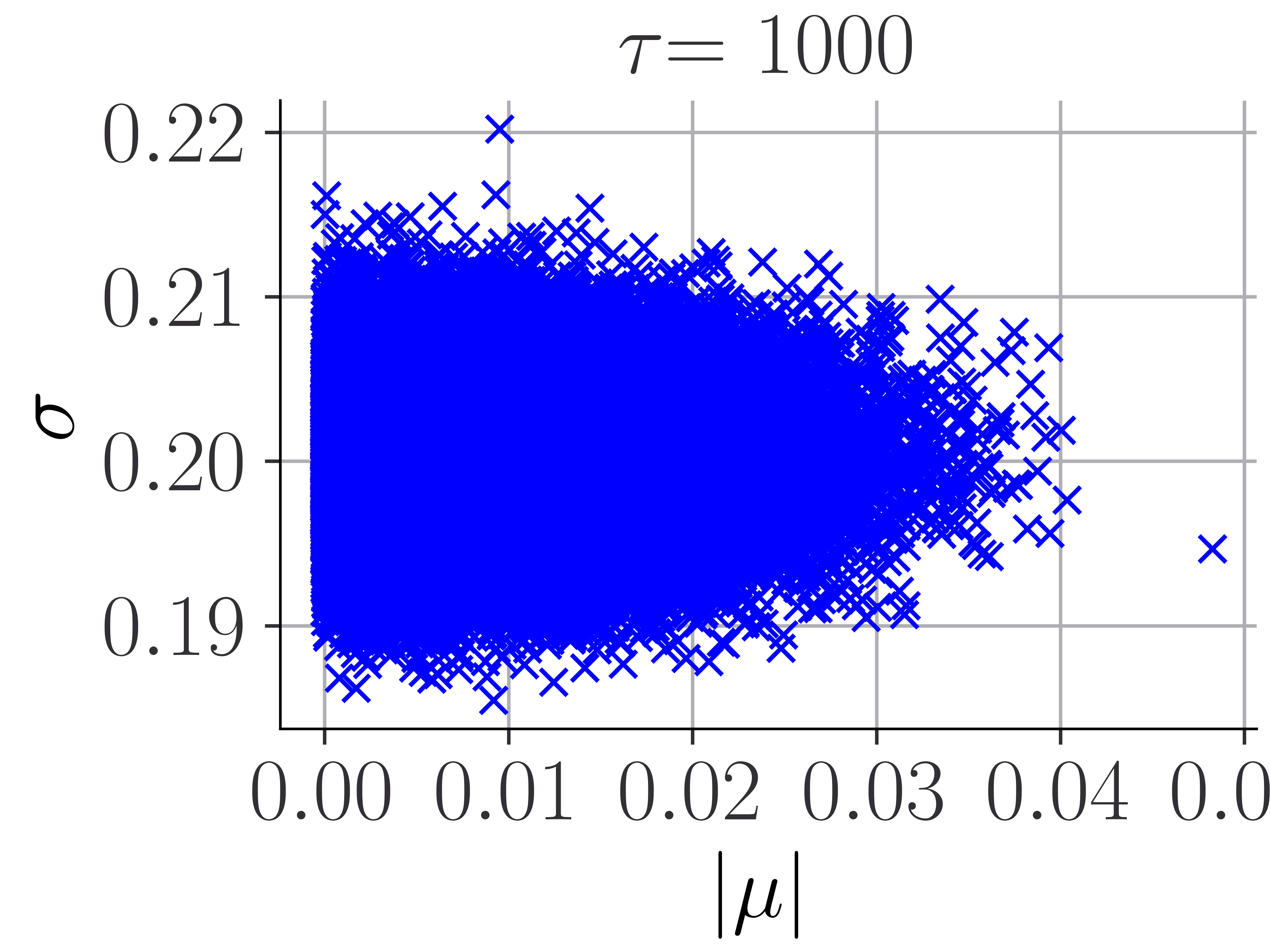}} \\
\subfigure{\includegraphics[width=0.22\columnwidth]{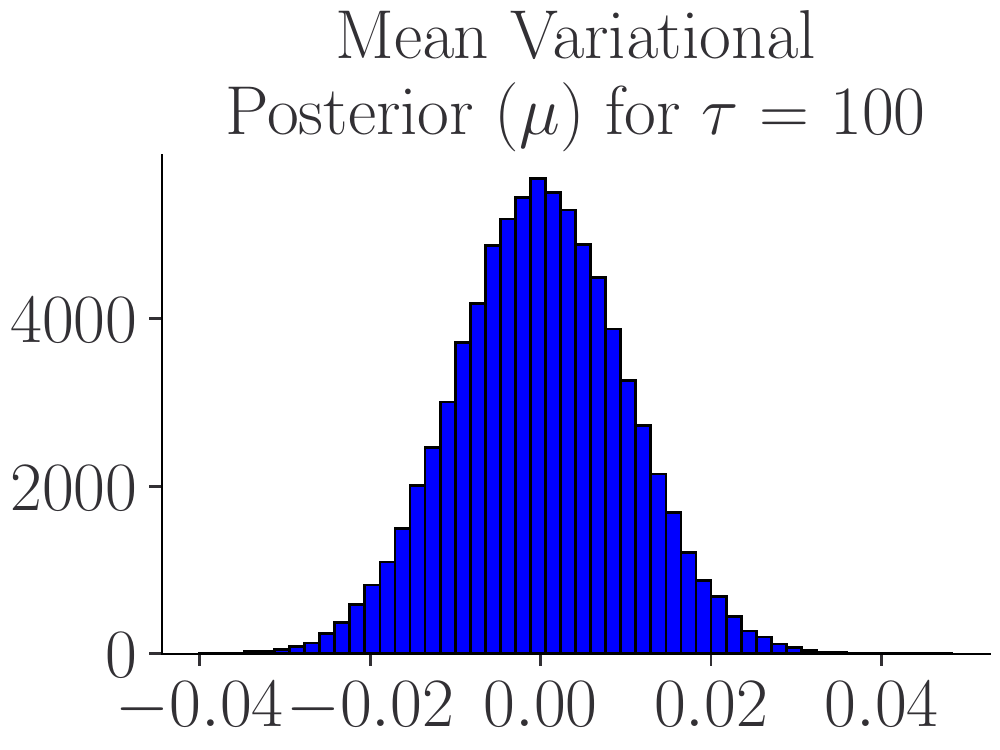}}
\subfigure{\includegraphics[width=0.22\columnwidth]{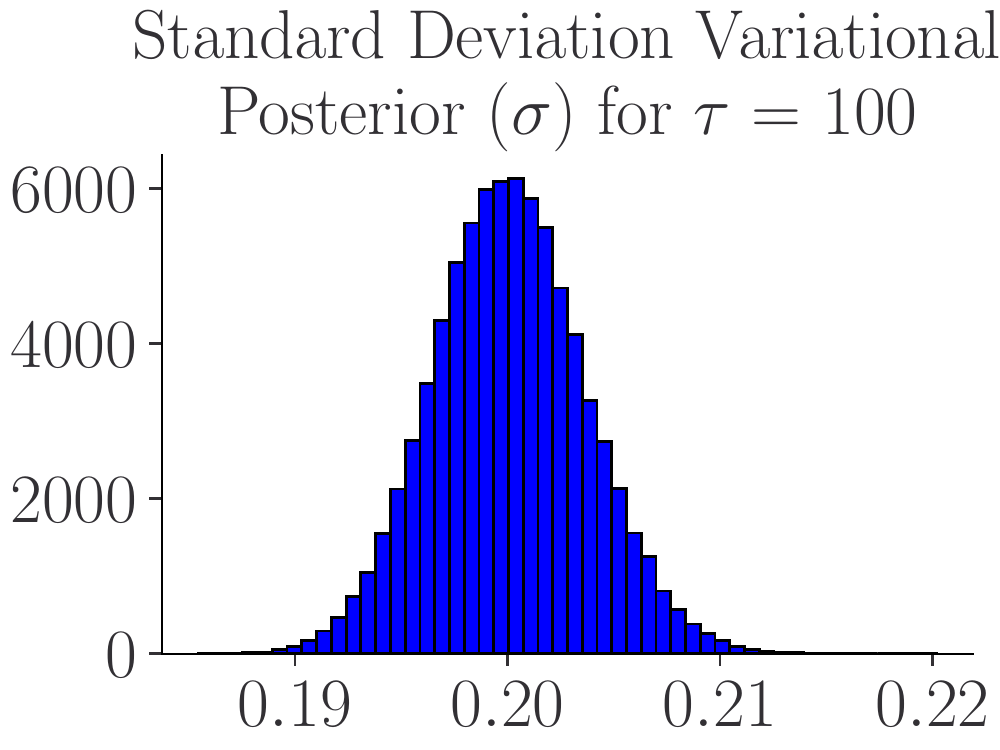}}
\subfigure{\includegraphics[width=0.22\columnwidth]{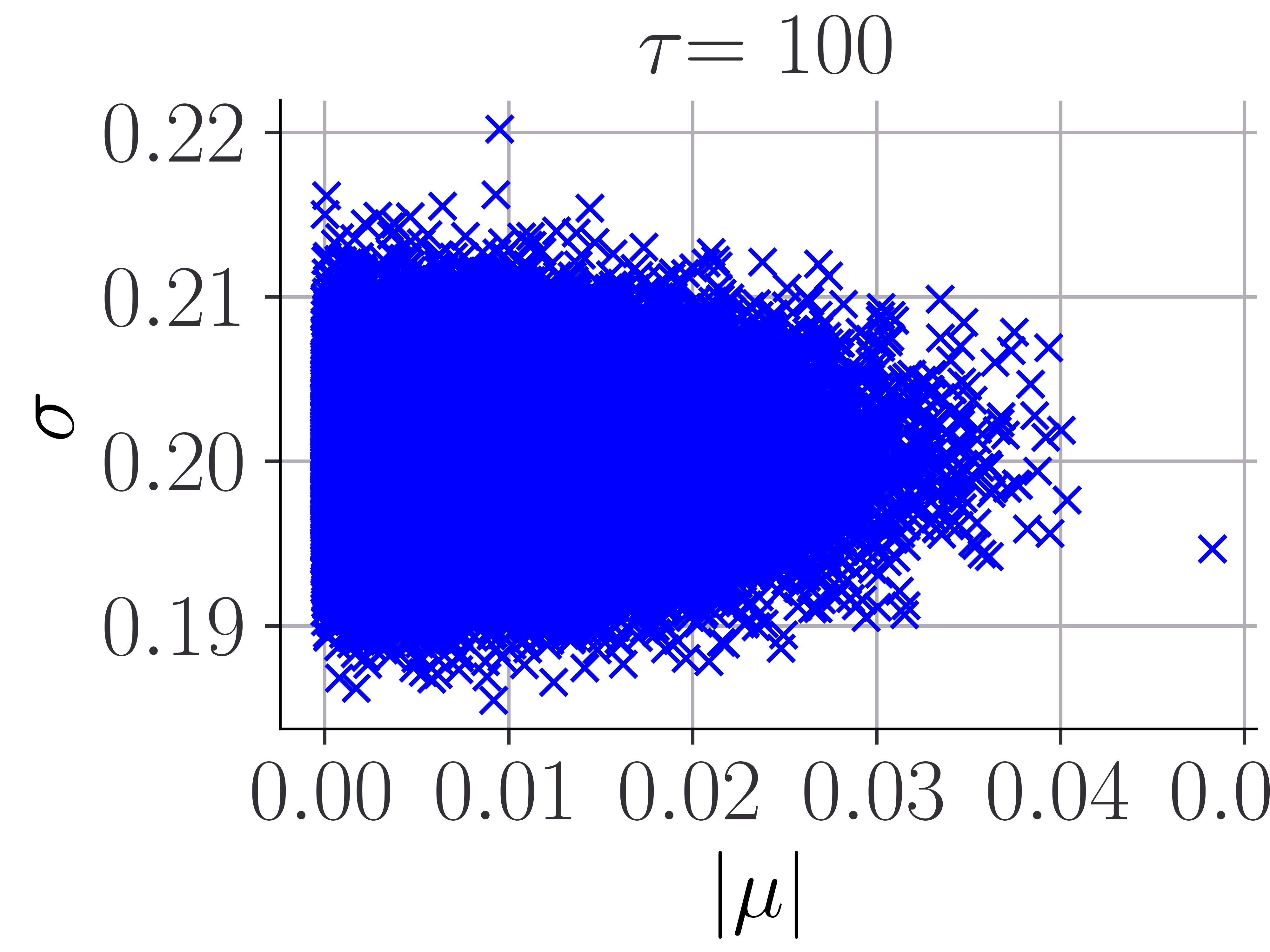}} \\
\subfigure{\includegraphics[width=0.22\columnwidth]{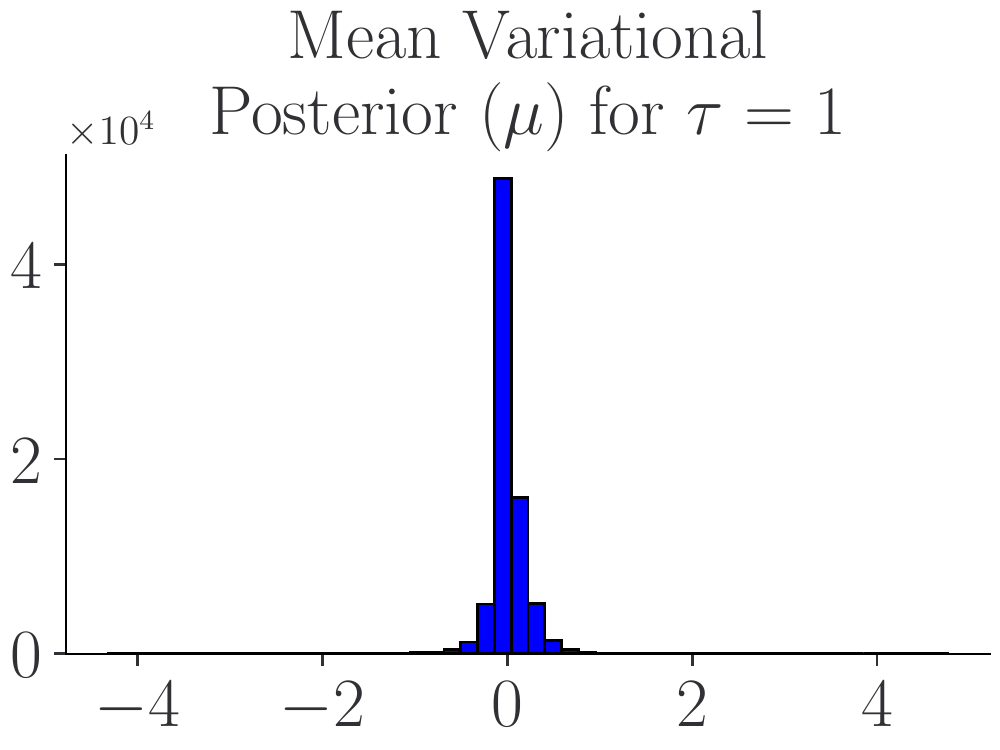}}
\subfigure{\includegraphics[width=0.22\columnwidth]{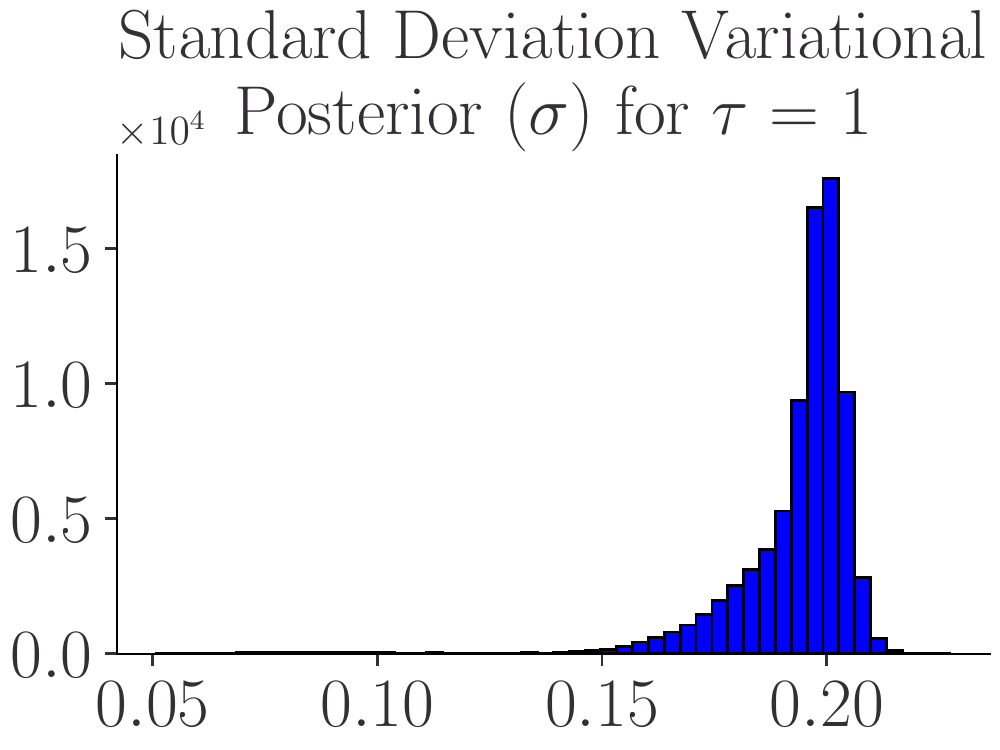}}
\subfigure{\includegraphics[width=0.22\columnwidth]{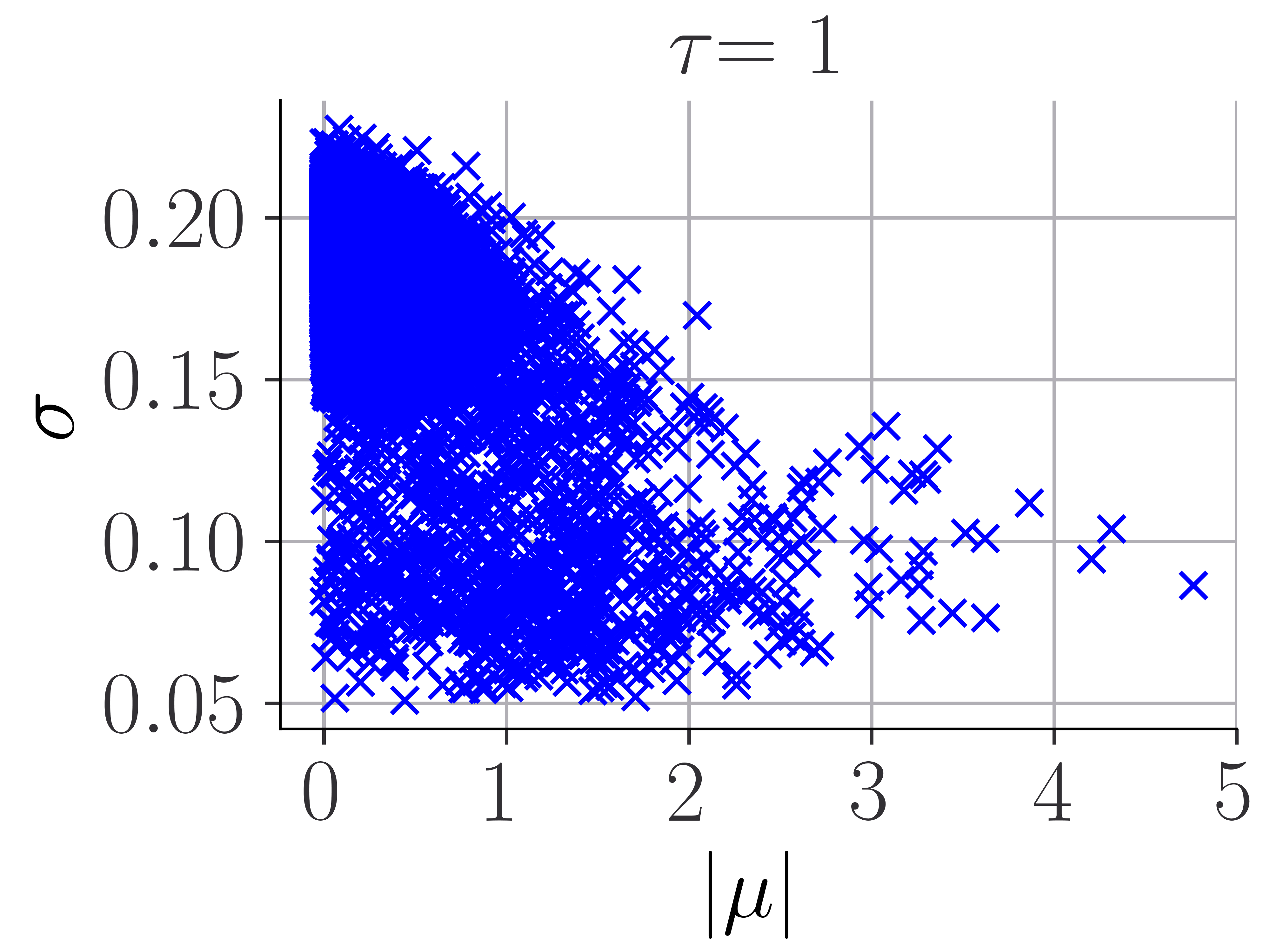}} \\
\subfigure{\includegraphics[width=0.22\columnwidth]{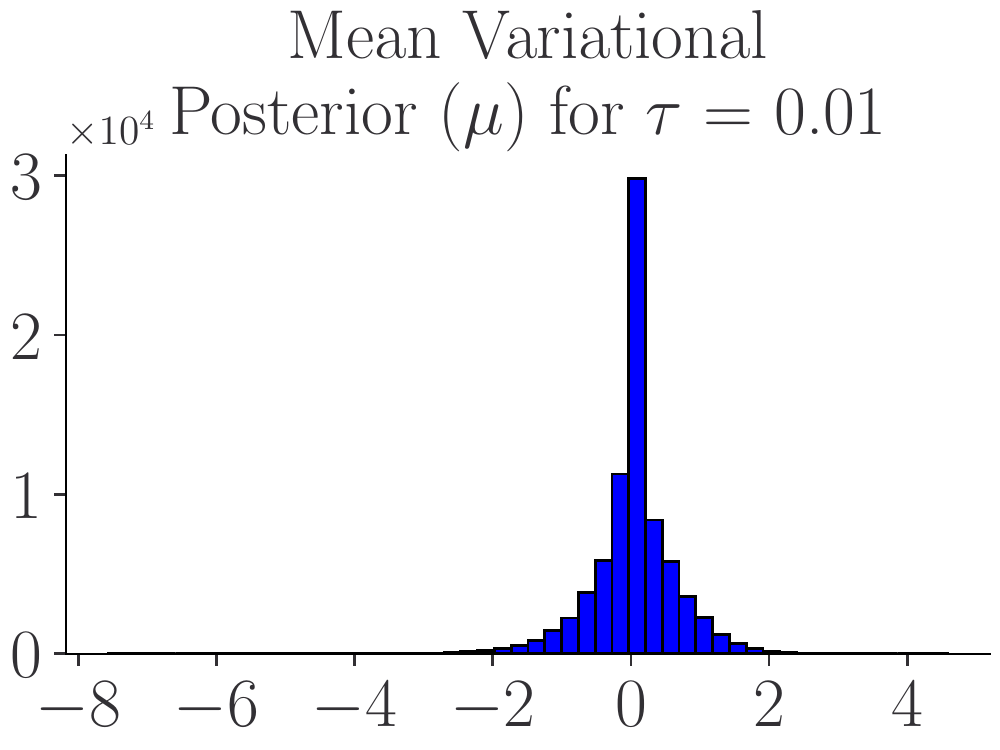}}
\subfigure{\includegraphics[width=0.22\columnwidth]{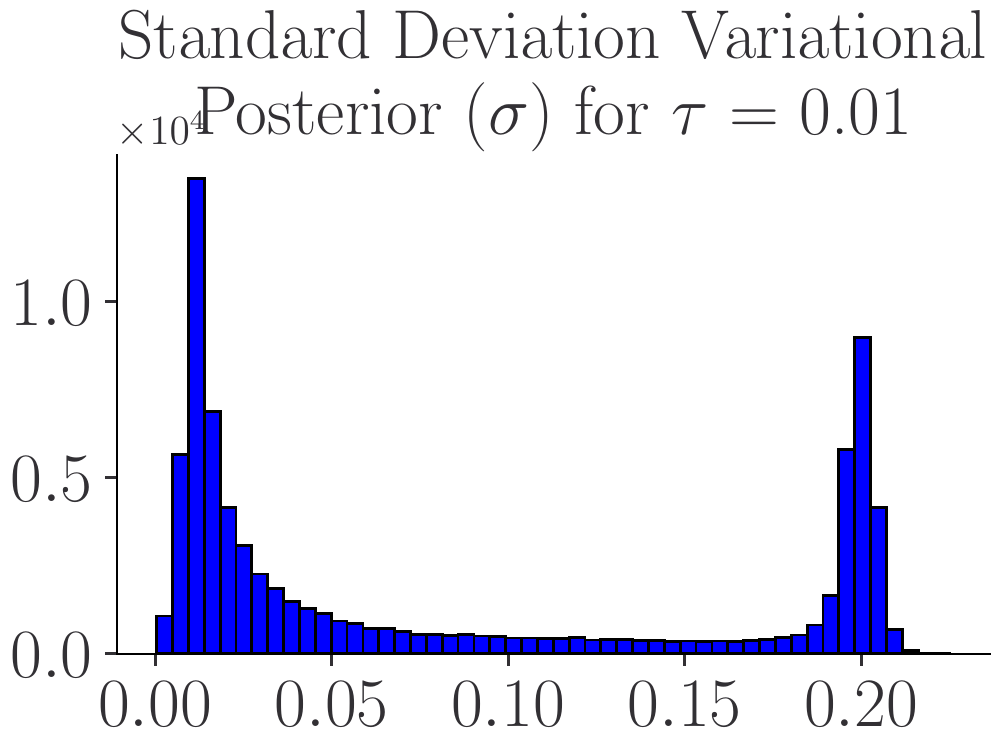}}
\subfigure{\includegraphics[width=0.22\columnwidth]{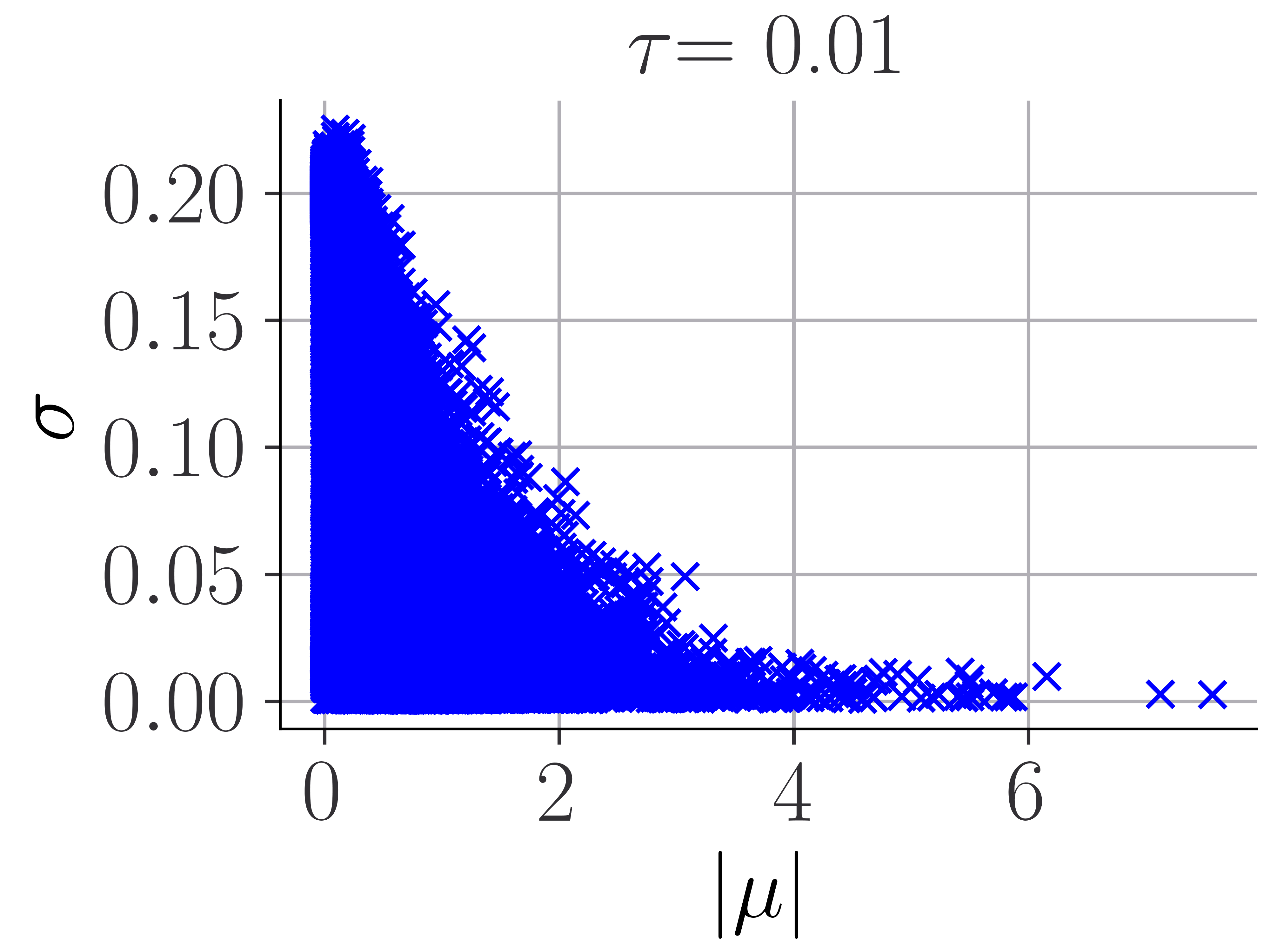}} \\
\caption{Histograms of the variational parameters $\theta=(\mu,\sigma)$ for a Linear BNN trained on MNIST.
From left to right: histogram of variational means, standard deviations, and standard deviation as a function of the norm of the mean.}
\label{fig:weights}
\end{figure}

\subsection{OOD detection}
We also compare the performance on out-of-distribution of a Resnet20 trained on CIFAR-10 with Bayes by Backprop. We compute the the histogram of predictive entropies for $5000$ in-distribution samples and out-of-distribution samples. Recall that the negative entropy is defined for a vector of class probabilities $[p(y=c|x,\cD)]_{c\in \{1,\dots,n_l\}}$ as $-\sum_{c=1}^{n_l}p(y=c|x,\cD) \log(p(y=c|x,\cD))$.
The first ones correspond to samples from the test set of CIFAR-10; while the out-of-distribution samples are chosen from another image dataset, namely SVHN \cite{netzer2011reading}. Our results are to be found in \Cref{fig:ood} and illustrate again the importance of the parameter $\tau$.
When $\tau$ is very small, the model is highly confident for in-distribution samples, and has diffuse predictive entropies for out-distribution samples. As $\tau$ increases, the model starts to be less confident, resulting in higher entropies on both in-distribution and out-distribution samples, especially for the out-distribution samples. Finally if $\tau$ is too large, as the model sticks to the prior distribution, it is not confident neither on the in-distribution nor out-distribution, resulting on a spiky distribution of predictive entropies at high values.

\begin{figure}[H]
\centering
\begin{tabularx}{\columnwidth}{ccc}
	\includegraphics[width=.3\columnwidth]{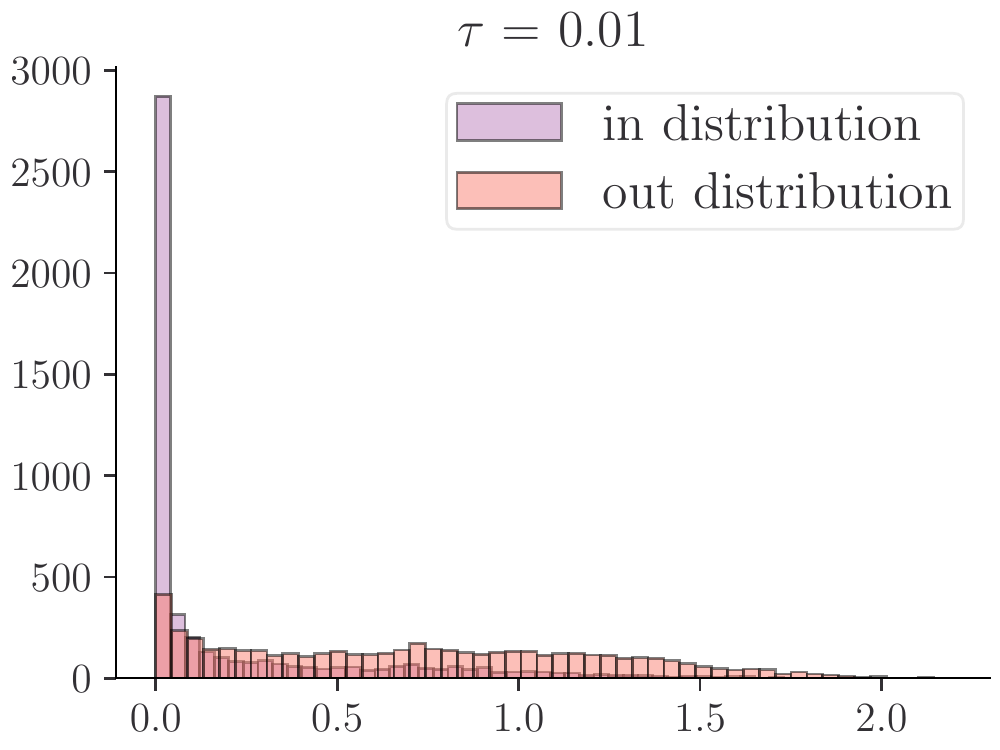}&
		\includegraphics[width=.3\columnwidth]{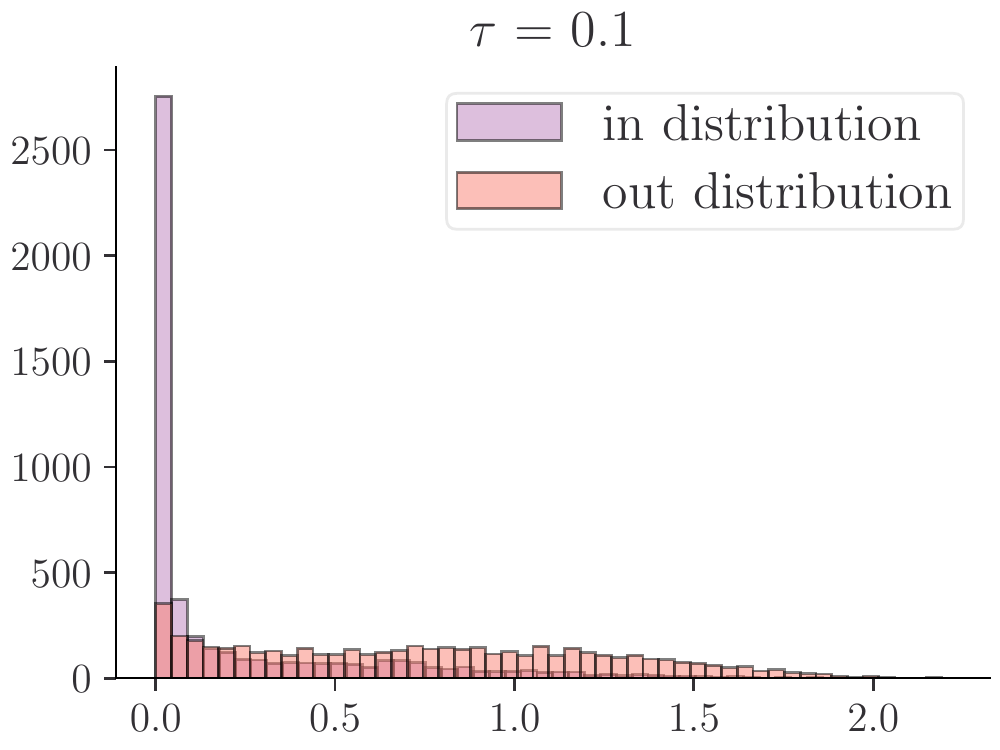}&
                \includegraphics[width=.3\columnwidth]{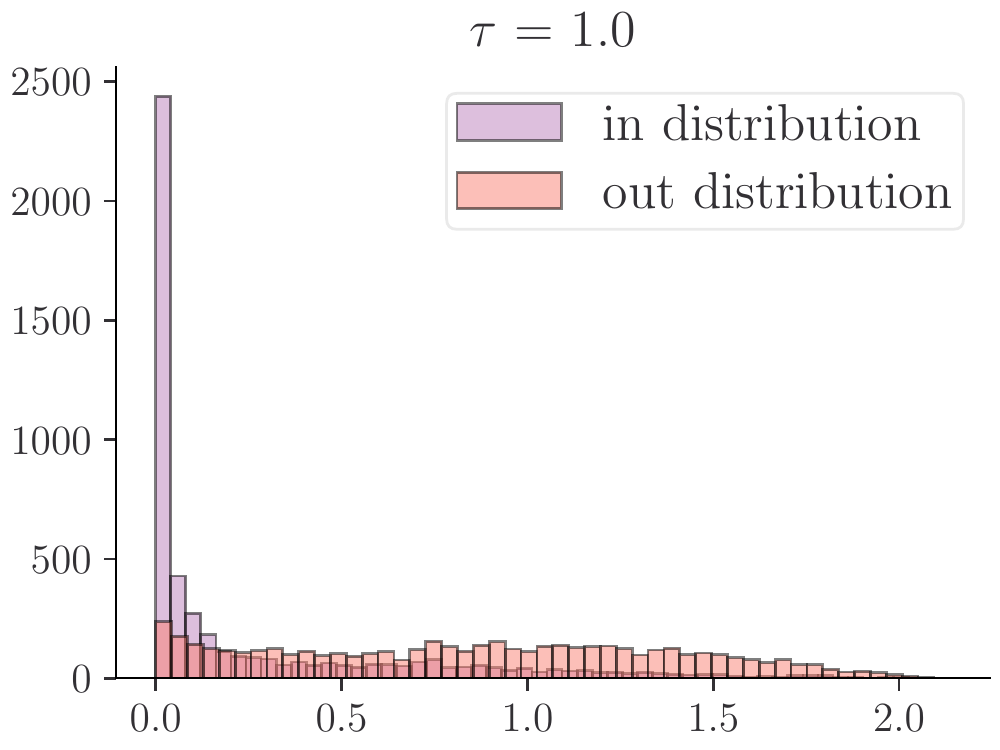}\\
  \includegraphics[width=0.3\columnwidth]{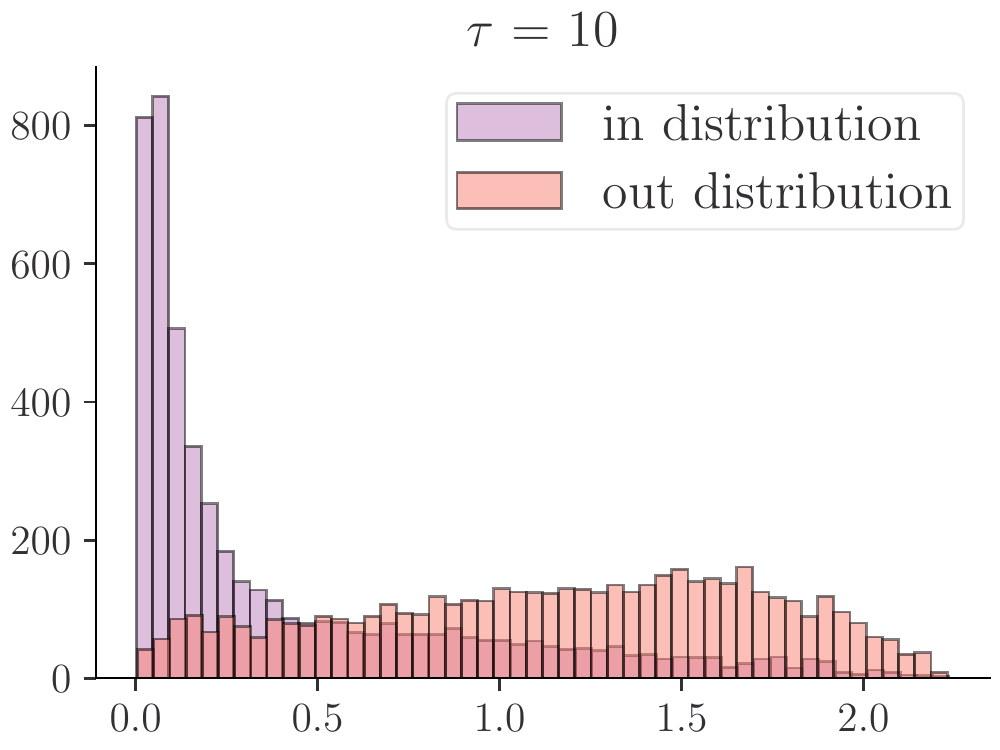} &\includegraphics[width=.3\columnwidth]{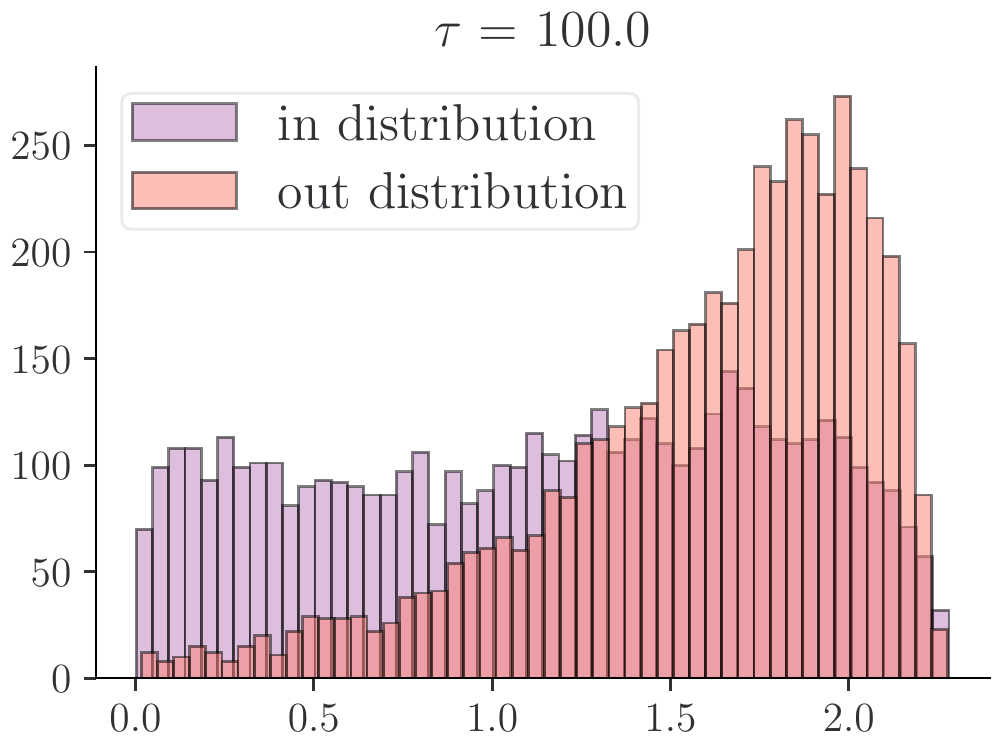} &\includegraphics[width=.3\columnwidth]{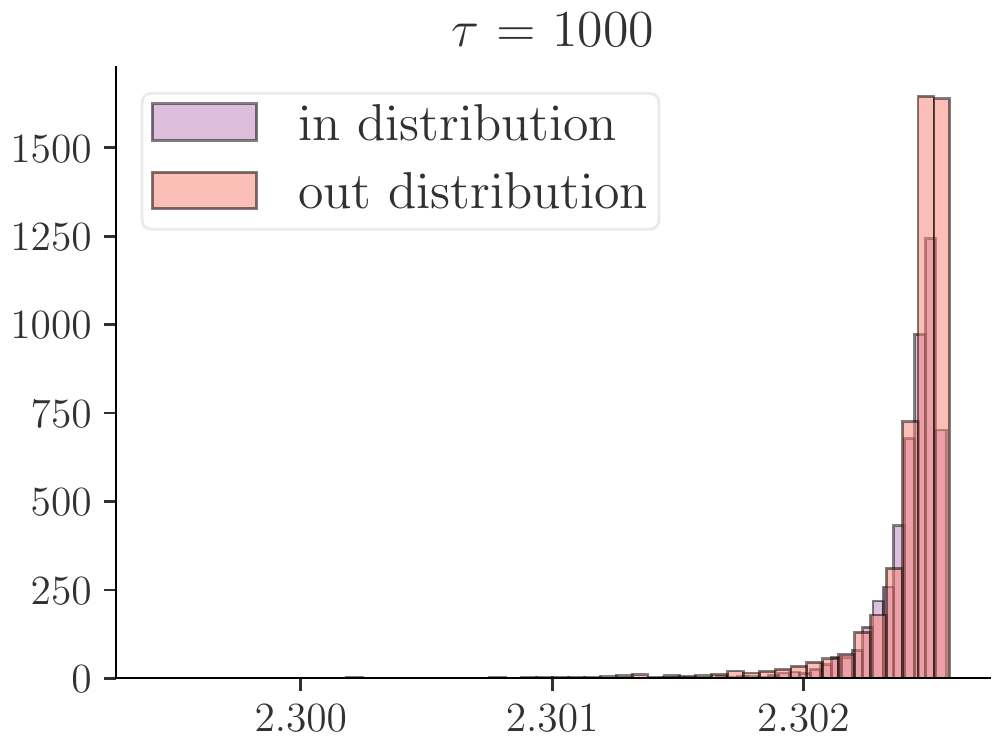}
\end{tabularx}
    \caption{
    Histogram of the predictive entropies for a Resnet20 trained on CIFAR-10, on $5000$ in-distribution (from the test set of CIFAR-10 dataset) and  out-of-distribution (from SVHN dataset) samples }
    \label{fig:ood}
\end{figure}

\section{Bayes by Backprop} \label{sec:bayes_by_backprop}
Several methods have been proposed to optimize $\ELBO$. A first and straightforward
approach is to apply stochastic gradient descent (SGD), using samples from $q_{\btheta_{\mathrm{c}}}$ where $\btheta_{\mathrm{c}}$ is the current point, to obtain stochastic estimates for $\nabla_{\theta} \ELBO$.
 However, the resulting estimation of the gradient suffers
from high variance. Alternative algorithms have been proposed to mitigate this effect, such as
Probabilistic Backpropagation \cite{hernandez2015probabilistic} or
Bayes by Backprop \cite{blundell2015weight}. Given a fixed distribution $\bgamma$ and a parameterized
function $g(\btheta,\cdot)$, the network parameter $\bw$ is obtained
as $\bw=g(\btheta,\bz)$, where $\bz$ is sampled from $\bgamma$, e.g.,
from a standard normal distribution. While a new $\bz$ is sampled at
each iteration, its distribution is constant, unlike that of the
network parameters $\bw$. As soon as $g(\btheta, \cdot)$ is invertible and $\boldsymbol{\gamma}, q(\cdot|\btheta)$ are non-degenerated probability distributions, we have $q(\bw|\btheta)d\bw= \bgamma(\bz)d\bz$ 
(see \citet[Appendix A]{jospin2020hands}), and for any differentiable function $f$:
\begin{equation}
    \frac{\partial}{\partial\btheta} \E_{\bw\sim q(.|\btheta)}[f(\bw,\btheta)] = \E_{\bz\sim \bgamma}\left[ \frac{\partial f(\bw,\btheta)}{\partial \btheta} +\frac{\partial \bw}{\partial \btheta} \frac{\partial f(\bw,\btheta)}{\partial \bw}  \right].
\end{equation}

\begin{algorithm}[h]
   \caption{Bayes by Backprop}
   \label{alg:bayes_by_backprop}
\begin{algorithmic}
\STATE \textbf{Input:} step-size $\delta>0$, number of iterations $m_{iter}$, number of samples $M_{samples}$.
\FOR{ each $m_{iter}$ iterations}
\FOR{ each  $m=1,\dots, M_{samples}$}
\STATE 1. Sample $\bz \sim \gamma^{\otimes N}$
\STATE 2. Let $\bw=\bmu+\log(1+\exp(\brho))\circ \bz$.
\ENDFOR
\STATE 3. Compute
\begin{equation}
    g(\bw,\btheta)\approx \frac{1}{M_{samples}}\sum_{m=1}^{M_{batch}} \log q(\bw_i|\btheta) - \log \prior(\bw_i)P(\cD|\bw_i)
\end{equation}
\STATE 5. Calculate the gradient with respect to the mean and  standard deviation parameter $\rho$
\begin{align}
    \Delta_{\mu}&=\frac{\partial g(w,\theta)}{\partial w} +\frac{\partial g(w,\theta)}{\partial \mu}\\
    \Delta_{\rho} &=\frac{\partial g(w,\theta)}{\partial w}\frac{\epsilon}{1+\exp(\rho)} +\frac{\partial g(w,\theta)}{\partial \rho}
\end{align}
\STATE 6. Update the variational parameters:
\begin{align}
    \mu \longleftarrow \mu - \delta \Delta_{\mu}\\
    \rho \longleftarrow \rho - \delta \Delta_{\rho}
\end{align}
\ENDFOR
\end{algorithmic}
\end{algorithm}
Bayes by Backprop uses the previous equality to estimate the gradient of $F$, because $F=\E_{\bw\sim q(\cdot|\btheta)}[f(\bw,\btheta)]$ with $f(\bw, \btheta) = \log q(\bw|\btheta) - \log \prior(\bw) -\log P(\D|\bw)$. More specifically, it performs a stochastic gradient descent for $F$ using a new sample $\bz$ at each time step to estimate the gradient of $F$ as the parameter $\btheta$ is updated. When the step size in this algorithm goes to zero, the Bayes by Backprop dynamics corresponds to a Wasserstein gradient flow of a particular functional defined on the space of probability distributions over $\btheta$, which we introduce in the next section.

As in \cite{blundell2015weight}, we will use a variance reparameterization; $\sigma=\log(1+\exp(\rho)) \in \R^{+}$ for $\rho \in \R$. Consequently, the variational parameter is given by $\btheta=(\theta_1,\dots, \theta_N)\in \R^{N\times 2d}$ with $\theta_j=(\mu_j, \rho_j)\in \R^{2d}$. We denote by $g:\R^{2d}\times \R^{d}\to \R, (\theta,z)\mapsto \mu + \log(1+\exp(\rho))\odot z$, where $\odot$ denotes the entry-wise multiplication and $\gamma$ denotes the standard normal distribution over $\R^d$.
The Bayes-by-backprop algorithm in this setting is summarized in \Cref{alg:bayes_by_backprop}.

This algorithm is well suited for minibatch optimisation, when the dataset $\cD$ is split into a partition of $L$ subsets (minibatches) $\cD_1,\dots,\cD_L$. In this case \cite{graves2011practical} proposes to minimise a rescaled $\NELBO$ for each minibatch $\cD_l$, $l=1,\dots,L$ as
\begin{equation}
  \NELBO_l = \frac{1}{L}\KL(\densityq_{\btheta}|\prior) - \E_{\bw\sim \densityq_{\btheta}}[\log P(\cD_l|\bw)].
\end{equation}



\end{document}